\documentclass{article}

\usepackage{microtype}
\usepackage{graphicx}
\usepackage{subfigure}
\usepackage{booktabs} % for professional tables

\usepackage{hyperref}

\usepackage[accepted]{icml2023_cr/icml2023}

\usepackage{amsmath}
\usepackage{amssymb}
\usepackage{mathtools}
\usepackage{amsthm}

\usepackage[capitalize,noabbrev]{cleveref}

\usepackage{amsfonts}
\usepackage{comment}
\usepackage{esint}
\usepackage{natbib}
\usepackage{multirow}
\usepackage[export]{adjustbox}

\definecolor{blue}{RGB}{0, 0, 0}
\newcommand{\st}[1]{}

\theoremstyle{plain}
\newtheorem{theorem}{Theorem}[section]
\newtheorem{proposition}[theorem]{Proposition}
\newtheorem{lemma}[theorem]{Lemma}

\theoremstyle{definition}
\newtheorem{definition}[theorem]{Definition}
\newtheorem{assumption}[theorem]{Assumption}
\theoremstyle{remark}
\newtheorem{remark}[theorem]{Remark}

\newcommand{\mb}[1]{\boldsymbol{#1}}
\newcommand{\bm}[1]{\mathbf{#1}}
\newcommand{\br}[1]{\left( #1 \right)}
\newcommand{\brs}[1]{\left[ #1 \right]}
\newcommand{\brangle}[1]{\left\langle #1 \right\rangle}
\newcommand{\brc}[1]{\left\{ #1 \right\}}

\DeclareMathOperator*{\argmin}{arg\,min}

\graphicspath{{figures/}}
\newcommand{\R}{\mathbb{R}}
\newcommand{\E}{\mathbb{E}}

\newcommand{\T}{^\top}

\newcommand{\defeq}{\vcentcolon=}

\newcommand{\hyv}{Hyv{\"a}rinen}
\newcommand{\p}{p_{\mb\theta}}   % p(x;theta)
\newcommand{\pbar}{\bar{p}_{\mb\theta}}   % p(x;theta)
\newcommand\myeq[1]{\stackrel{\mathclap{\tiny\mbox{#1}}}{=}}

\newcommand\myleq[1]{\stackrel{\mathclap{\tiny\mbox{#1}}}{\leq}}
\newcommand{\kxxp}[1]{\mb{\varphi}_{\mb{#1}, \bm{x}'}}

\newcommand{\dy}{\partial_{y_l}}
\newcommand{\dx}{\partial_{x_l}}
\newcommand{\dz}{\partial_{z_l}}

\newcommand{\score}[1]{\mb{\psi}_{#1}}
\newcommand{\scorel}[1]{\psi_{#1, l}}
\newcommand{\Eq}[2]{\E_{\mb{#1} \sim #2}}
\newcommand{\Eqsmall}[2]{\E_{#2}}

\newcommand{\approxdV}{\widetilde{\partial V}}
\newcommand{\approxG}{\mathcal{G}^d_{0, m}}

\newcommand{\latentscorexz}{\psi_{q}(\mb{x} | \mb{z})}
\newcommand{\latentscorexzj}{\psi_{q, j}(\mb{x} | \mb{z})}
\newcommand{\latentscorexzji}{\psi_{q, j}(\mb{x} | \mb{z}_i)}

\usepackage[disable,textsize=tiny]{todonotes}
\icmltitlerunning{Approximate Stein Classes for Truncated Density Estimation}

\begin{document}

\twocolumn[
\icmltitle{Approximate Stein Classes for Truncated Density Estimation}
\icmltitlerunning{Approximate Stein Classes for Truncated Density Estimation}

\icmlsetsymbol{equal}{*}

\begin{icmlauthorlist}
\icmlauthor{Daniel J. Williams}{yyy}
\icmlauthor{Song Liu}{yyy}
\end{icmlauthorlist}

\icmlaffiliation{yyy}{School of Mathematics, University of Bristol, UK}

\icmlcorrespondingauthor{Daniel J. Williams}{daniel.williams@bristol.ac.uk}

\icmlkeywords{Density Estimation, Truncated Densities, Stein Discrepancies, Score Matching, Parameter Estimation}

\vskip 0.3in
]



\printAffiliationsAndNotice{}  % leave blank if no need to mention equal contribution

\begin{abstract}

Estimating truncated density models is difficult, as these models have intractable normalising constants and hard to satisfy boundary conditions. 
Score matching can be adapted to solve the truncated density estimation problem, but requires a continuous weighting function which takes zero at the boundary and is positive elsewhere. 
Evaluation of such a weighting function (and its gradient) often requires a closed-form expression of the truncation boundary and finding a solution to a complicated optimisation problem. 
In this paper, we propose \textit{approximate Stein classes}, which in turn leads to a relaxed Stein identity for truncated density estimation. 
We develop a novel discrepancy measure, \textit{truncated kernelised Stein discrepancy} (TKSD), which does not require fixing a weighting function in advance, and can be evaluated using only samples on the boundary.
We estimate a truncated density model by minimising the Lagrangian dual of TKSD. 
Finally, experiments show the accuracy of our method to be an improvement over previous works even without the explicit functional form of the boundary.

\end{abstract}

\section{Introduction}
% \st{
In truncated density estimation, we are unable to view a full picture of our dataset. We are instead given access to a smaller subsample of data artificially truncated by a boundary. Examples of truncation boundaries include limited number of medical tests resulting in under-reported disease counts, and a country's borders preventing discoveries of habitat locations. In either case, a complex boundary causes truncation which introduces difficulties in statistical parameter estimation. Regular estimation techniques such as maximum likelihood estimation (MLE) are ill-suited, as we will explain. 
% }
% {\color{blue}
% In truncated density estimation, we are unable to view a full picture of our dataset. We are instead given access to a smaller subsample of data artificially truncated by a boundary, 
% for example, limited number of medical tests giving way to under-reported disease count. 
% As a further example, suppose we want to estimate a density model for the distribution of submissions at a machine learning conference. We can only observe the accepted papers, as the rejected ones are not visible to us. 
% Moreover, it is hard to provide a clear definition of a truncation boundary that distinguishes between the accepted and rejected submissions, but we can easily identify borderline submissions from the review scores. In fact, we presume that it is not difficult for human experts in other fields to provide a complete definition of the boundary, and it is much easier to provide boundary examples. Under this setting, we have only a collection of finite samples to define the truncation boundary, rather than a functional form, to which no current truncated density estimation method is applicable. 
% }

When data are truncated, many typical statistical assumptions break down. One fundamental problem with estimation in a truncated space is that the probability density function (PDF), given by
\[
\p(\mb{x}) = \frac{\pbar(\mb{x})}{Z(\mb\theta)},  \; \; Z(\mb\theta) = \int_V \pbar(\mb{x}) d\mb{x},
\]
cannot be fully evaluated. In this setup, $V$ is the truncated domain which can be highly complex and thus the integration to obtain the normalising constant, $Z(\mb\theta)$, is intractable. This normalising constant could be approximated via numerical integration, such as with Monte Carlo methods \citep{kalos2009monte}, but this comes with a significant computational expense. Recent attention has turned to estimation methods which bypass the calculation of the normalising constant entirely by working with the \textit{score function} for $\p(\mb{x})$,
\[
\score{\p} = \score{\p}(\mb{x}) := \mb\nabla_{\mb{x}} \log \p(\mb{x}) = \mb\nabla_{\mb{x}} \log \pbar(\mb{x}).
\]
which uniquely represents a probability distribution. Score-based estimation methods include score matching \citep{scorematching1, scorematching2}, noise-contrastive estimation \citep{gutmann2010, gutmann2012} and minimum Stein discrepancies \citep{stein1972, barp2019}.
% and kernelised Stein discrepancies (KSD) \citep{kernelgoodness, ksd}. 
These methods are computationally fast and accurate, and usually rely on minimising a discrepancy between the score functions for the model density $\p(\mb{x})$ and the unknown data density $q(\mb{x})$, whose score function is $\score{q}: = \mb\nabla_{\mb{x}}\log q(\mb{x})$.

Score-based methods have been applied across many domains, including hypothesis testing \citep{ksd, kernelgoodness, bdksd, Wu2022}, generative modelling \citep{song2019generative, song2020score, pang2020}, energy based modelling \citep{trainebm}, and Bayesian posterior estimation \citep{jack}.
Recently, 
two lines of work
% two methods 
have been proposed for the truncated domain; truncated density estimation via score matching \citep{yu2021, song, williams2022}, called \textit{TruncSM}, and truncated goodness-of-fit testing via the kernelised Stein discrepancy (KSD), denoted bounded-domain KSD (bd-KSD) \citep{bdksd}.

These two lines of work both use a distance function as a weighting function on the objective. Such a function is chosen in advance such that the boundary conditions required for deriving score matching or KSD hold when the domain is truncated. The computation of this distance function can be challenging when the boundary is complex and high-dimensional, which we will demonstrate later in experiments. Further, these methods rely on knowing a functional form of the boundary, which is not always available.
% In this case, it is required to calculate the distance from any point to the boundary, which is a challenging optimisation problem when the boundary is complex or high-dimensional. We demonstrate this problem in later experiments. 

%Our work aims to derive an estimator for truncated probability densities which is adaptive and data-driven. 

% \begin{figure}[t!]
%     \centering
%     \includegraphics[width=\linewidth]{newfigures/different_g_mu0.pdf}
%     \includegraphics[width=\linewidth]{newfigures/different_g_mu0_75.pdf}
%     \caption{\todo{add labels}An example of the Stein discriminatory function for TKSD (ours) and bd-KSD, and the weighting function for \textit{TruncSM}, for an exemplary standard Gaussian dataset truncated on $[-1, 1]$. When $\mb{\hat{\mu}}$, the estimated mean of the distribution, is manually changed, the shape of the \textit{TruncSM} weighting function remains unchanged, whereas the bd-KSD and TKSD discriminatory functions are adaptive to $\mb{\hat{\mu}}$, whereas bd-KSD is still restricted by a similar weighting function to \textit{TruncSM}.}
%     \label{fig:different_g}
% \end{figure}

In this paper, we consider a situation where the functional form of the boundary is not available to us. We can only access the boundary information through \emph{a finite set of random samples}. As an example, suppose we want to estimate a density model for submissions at an academic conference. We can only observe the accepted papers but not the rejected ones. Furthermore, it is difficult to provide a functional definition of the truncation boundary that distinguishes between the accepted and rejected submissions. However, we can easily identify \emph{borderline} submissions from the review scores which can be seen as samples of the boundary. 
% In fact, we presume that it is not difficult for human experts in other fields to provide a complete definition of the boundary, and it is much easier to provide boundary examples.
{\color{blue} In this case, \textit{TruncSM} and bd-KSD are not applicable due to the lack of a functional definition of the boundary, and classical methods such as MLE are intractable.} \st{The weighting functions from \textit{TruncSM} and bd-KSD cannot be specified using this approximate boundary, and classical approaches such as MLE cannot be used either due to the lack of tractable normalising constant.} To our knowledge, there exists no method to estimate a truncated density when there is no functional form of the boundary. To solve this problem, we first define \textit{approximate Stein classes} and its corresponding Stein discrepancies, which we refer to as truncated kernelised Stein discrepancy (TKSD), which is computationally tractable with only samples from the boundary. By minimising the TKSD, we obtain a truncated density estimator when the boundary's functional form is unavailable. In experiments, we show that despite the approximate nature of the Stein class, density models can be accurately estimated from truncated observations. We also provide a theoretical justification of the estimator consistency.

Our main contributions are:
\begin{itemize}
\itemsep0em
    \item We introduce approximate Stein classes, which in turn define approximate Stein discrepancies. Unlike earlier approaches, these Stein discrepancies are more relaxed and applicable to a truncated setting.
    \item These discrepancies enable density estimation on truncated datasets. This estimator is an extension of earlier Stein discrepancy estimators. We include theoretical and experimental results showing that the TKSD estimator is a consistent and competitive estimator with previous works.
\end{itemize}

% Our work focuses on adapting the regular KSD estimator such that its discriminatory function serves as a non-parametric solution to the weighting function, without the need to specify its form in advance. 
% See \Cref{fig:different_g} for a comparison between shapes of our discriminatory function, and the equivalent for \textit{TruncSM} and bd-KSD. 
% This discriminatory function does not require a distance calculation, which can be computationally expensive or difficult to obtain for complex truncation boundaries. Instead, truncated KSD (TKSD) only requires a finite set of samples from the boundary, where our constraint is enforced, which is straightforward to obtain for common cases. This gives an estimator which is flexible based on the dataset at hand, adaptive to different boundaries, and computationally insensitive to the complexity of the truncation boundary.

\section{Problem Setup}

\textbf{The Problem.} \quad We assume the data density, $q$, has support on $V \subseteq \R^d$, whose boundary is denoted as $\partial V$. We aim to find a model, given by $\p$, which best estimates $q$. However, there are some significant challenges for the truncated setting: the normalising constant in $\p$ is intractable, and score-based methods rely on a boundary condition which breaks down in this situation. 
Previous works \citep{bdksd, song} do address these issues. 
However, in this work, we assume we have no functional form of $\partial V$, and instead, a finite set of samples, $\{\mb{x}'_i\}^m_{i=1}$, which are drawn randomly from the boundary. To our best knowledge, no existing density estimation methods can be applied directly. 
% Previous works \citep{bdksd, song} address the first two issues, but there is no solution for when the functional form of the boundary is unknown.

%problem: two issues: normalising constant and functional boundary. can use SM etc for first one, but second one has no solution, mention {x_i}^m
% two layers of

\textbf{The Aims.} \quad We aim to use unnormalised models to estimate a truncated density, bypassing the evaluation of the normalising constant. We also aim to construct an estimator that does not require a functional form of the boundary,
% $\partial V$
but can still adjust to the boundary and the dataset adaptively. This will lead to an estimator which is \emph{data-driven} and \emph{flexible}, not relying on a pre-defined weighting function, unlike previous works by \citet{bdksd} and \citet{song}. 
% To do so, we must create a new framework of Stein discrepancy estimators which do not rely on the boundary condition described above.

% To explain our methods we need background methods on score based methodologies - for SM, KSD, etc.
Before explaining our proposed solution, we introduce prior methods for measuring discrepancies for unnormalised densities.

\section{Background}

Statistical density estimation is often performed by minimising a divergence between the data density, $q(\mb{x})$, and the model density, $\p(\mb{x})$. Since truncated densities have intractable normalising constants, we introduce divergence measures suited for unnormalised densities and discuss their generalizations to truncated supports. These methods rely on Stein's identity, which we will introduce foremost.

\subsection{Stein's Identity}

% Stein's method \citep{stein1972} was originally proposed for acquiring bounds on distances between a sequence of random variables and the normal distribution, but is now used as a general metric for bounding the distance between two probability distributions. Steins method relies heavily on 

Originating from Stein's method \citep{stein1972, chen2011}, a \textit{Stein class} of functions enables the construction of a family of discrepancy measures \citep{barp2019}.

\begin{definition}
    Let $q=q(\mb{x})$ be any smooth probability density supported on $\R^d$  and let $\mathcal{S}_{q}: \mathcal{F}^d \to \R$ be a map. $\mathcal{F}^d$ is a Stein class of functions, if for any $\mb{f} \in \mathcal{F}^d$,
    \begin{equation}
        \Eqsmall{x}{q} [ \mathcal{S}_{q} \mb{f}(\mb{x})] = 0,
    \label{eq:steinidentity}
    \end{equation}
    where $\mathcal{S}_q$ is called a Stein operator \citep{gorham2015}. 
\label{def:steinidentity}
\end{definition}

We refer to \cref{eq:steinidentity} as the Stein identity, which underpins a lot of existing work in unnormalised modelling. The Langevin Stein operator \citep{gorham2015} on $\p(\mb{x})$, 
\[
\mathcal{S}_{\p} \mb f(\mb{x}) = \mathcal{T}_{\p} \mb f(\mb{x}) := \sum^d_{l=1}\scorel{\p}(\mb{x}) f_l(\mb{x}) + \dx f_l(\mb{x}),
\]
where $\scorel{\p}(\mb{z}) := \dz \log \p(\mb{z})$, is independent of the normalising constant, $Z(\mb\theta)$. $\p$ is involved in the Stein operator only via its `proxy', the score function, $\score{\p}$. When this Langevin Stein operator is used, it is straightforward to see that \cref{eq:steinidentity} holds:
\begin{align}
    \Eqsmall{x}{q} [ \mathcal{T}_{q} \mb{f}(\mb{x})] &= \int_{\R^d} q(\mb{x})\Big({\sum^d_{l=1}\scorel{q} f_l(\mb{x}) + \dx f_l(\mb{x})}\Big)d \mb{x} \nonumber\\
    &= \sum^d_{l=1} \int_{\R^d}  \dx q(\mb{x}) f_l(\mb{x}) + q(\mb{x})\dx f_l(\mb{x})d \mb{x} \nonumber\\
    & = 0, \notag
    % \\
    % &= \oint p(\mb{x}) \mb{f}(\mb{x})d\mb{x} = 0. 
\end{align}
% Where the final surface integral 
where the last equality holds 
due to integration by parts and the fact that $q(\mb{x})$ vanishes at infinity: 
\begin{equation}
\lim_{\|\mb{x} \| \to \infty} q(\mb{x}) = 0. 
\label{eq:steinidentity_boundary}
\end{equation}
This assumption is critical, and is a key focus of research for this paper. It holds for many densities supported on $\R^d$, such as the Gaussian distribution or the Gaussian mixture distribution. In the rest of this paper, we refer to this condition as the boundary condition, as it describes the behaviour of $q(\mb{x})$ at the boundary of its domain. 

% A common example of this holding is if $p(\mb{x})$ were a Gaussian distribution.

\subsection{Divergence for Unnormalised Densities}

When $\mb{x} \in \R^d$, and the boundary condition \cref{eq:steinidentity_boundary} holds, we describe two computationally tractable discrepancy measures for unnormalised density models $p_{\bm \theta}(\mb{x})$: the Stein discrepancy and the score matching divergence. Both divergences rely on Stein's identity to derive a tractable form. 

\subsubsection{Classical Stein Discrepancy}

\citet{gorham2015} use Stein's identity to define a \textit{Stein discrepancy}: the supremum of the differences between expected Stein operators for two densities $q(\mb{x})$ and $\p(\mb{x})$,
\begin{align}
\mathcal{D}_{SD}(\p | q)  :&= \sup_{\mb{f} \in \mathcal{F}^d} \br{ \Eqsmall{x}{q}[\mathcal{T}_{\p} \mb{f}(\mb{x})] - \Eqsmall{x}{\p} [\mathcal{T}_{\p} \mb{f}(\mb{x})] } \nonumber \\
&= \sup_{\mb{f} \in \mathcal{F}^d}\Eqsmall{x}{q} [\mathcal{T}_{\p} \mb{f}(\mb{x})], \label{eq:sd}
\end{align}
where the second line holds as $\mathcal{F}^d$ is a Stein class.
% , the above supremum leads to the following equivalent definition: 
%
% \begin{definition} \label{def:steindiscrepancy}
%     Let $q=q(\mb{x})$ and $\p=\p(\mb{x})$ be two smooth densities supported on $\R^d$, and $\mb{x} \sim q$. The Stein discrepancy takes the form
%     \begin{equation}
%     \mathcal{D}_{SD}(\p | q) = \sup_{\mb{f} \in \mathcal{F}^d}\Eqsmall{x}{q} [\mathcal{T}_{\p} \mb{f}(\mb{x})],
%     \label{eq:sd}
%     \end{equation}
%     where $\mathcal{F}^d$ is a Stein class of $\p$.
% \end{definition}
%
The Stein discrepancy can be interpreted as the maximum violation of Stein's identity. $\mb{f}$ is referred to as the \textit{discriminatory function}, as it discriminates between $q(\mb{x})$ and $\p(\mb{x})$. However, the supremum in \cref{eq:sd} across $\mathcal{F}^d$ is a challenging problem for optimisation \citep{gorham2015}.

Stein discrepancies have seen a lot of recent development, including extensions to non-Euclidean domains \citep{shi2021, xu2021}, discrete operators \citep{yang2018}, stochastic operators \citep{Gorham2020} and diffusion-based operators \citep{gorham2015, Gorham2020}.

\subsubsection{Kernelised Stein Discrepancy (KSD)}
When we restrict the function class over which the supremum is taken to a Reproducing Kernel Hilbert Space (RKHS), 
we can derive the Kernelised Stein discrepancy (KSD), for which we follow a similar definition to \citet{kernelgoodness} and \citet{ksd}. 

First, let $\mathcal{G}$ be an RKHS equipped with positive definite kernel $k$, and let $\mathcal{G}^d$ denote the product RKHS with $d$ elements, where $\mb{g} = (g_1, \dots, g_d) \in \mathcal{G}^d$, and is defined with inner product $\langle \mb{g}, \mb{g'} \rangle_{\mathcal{G}^d} = \sum^d_{i=1}\langle g_i, g'_i\rangle_{\mathcal{G}}$ and norm $\| \mb{g} \|_{\mathcal{G}^d} = \sqrt{\sum^d_{i=1}\langle g_i, g_i\rangle_{\mathcal{G}}}$ . By the reproducing property, any evaluation of $\mb{g} \in \mathcal{G}^d$ can be written as
\begin{equation}
\mb{g}(\mb{x}) = \brangle{\mb{g}, k(\mb{x}, \cdot)}_{\mathcal{G}^d} = \sum^d_{i=1}\langle g_i, k(\mb{x}, \cdot)\rangle_{\mathcal{G}}.
\label{eq:rkhsevals}
\end{equation}
%Additionally, $\mathcal{H}$ is equipped with a norm $\|h\|_{\mathcal{H}} = \brangle{h, h}_{\mathcal{H}}$. 
Taking the supremum over $\mathcal{G}^d$, and including the restriction of $\mb{g}$ to the RKHS unit ball, i.e. $\|\mb{g}\|_{\mathcal{G}^d} \leq 1$, gives rise to the KSD 
\begin{align}
\mathcal{D}_{\text{KSD}}(\p | q) &:= \sup_{\substack{\mb{g} \in \mathcal{G}^d, \|\mb{g}\|_{\mathcal{G}^d} \leq 1}} \Eqsmall{x}{q} [\mathcal{T}_{\p} \mb{g}(\mb{x})] \nonumber \\
&= \|\Eqsmall{x}{q} [\mathcal{T}_{\p} k(\mb{x}, \cdot)]\|_{\mathcal{G}^d}.
\label{eq:ksd1}
\end{align}
The KSD has a closed-form expression as indicated by \cref{eq:ksd1}. 
Moreover, the squared KSD can be expanded to a double expectation
\begin{align}
\mathcal{D}_{\text{KSD}}(\p | q)^2 &= \|\Eqsmall{x}{q} [\mathcal{T}_{\p} k(\mb{x}, \cdot)]\|_{\mathcal{G}^d}^2 \nonumber \\
&= \Eq{x}{q} \Eq{y}{q} \brs{ \sum_{l=1}^d u_l(\mb{x}, \mb{y}) }
\label{eq:ksd3}
\end{align}
where 
\begin{align}
u_l(\mb{x}, \mb{y}) &=\scorel{p}(\mb{x})\scorel{p}(\mb{y})k(\mb{x},\mb{y}) +\scorel{p}(\mb{x}) \dy k(\mb{x}, \mb{y}) \nonumber \\
&+\scorel{p}(\mb{y}) \dx k(\mb{x}, \mb{y}) + \dx\dy k(\mb{x}, \mb{y}). \label{eq:ul}
\end{align}
This divergence can be fully evaluated using samples from $q(\mb{x})$ to approximate the expectation in \cref{eq:ksd3}. 
% The KSD has important properties that make it an ideal discrepancy for measuring differences between distributions, namely that under mild conditions, 
Further, \citet{kernelgoodness} showed that $\mathcal{D}_{\text{KSD}}(\p | q) = 0$ if and only if $\p = q$, making $\mathcal{D}_{\text{KSD}}(\p | q)$ a good discrepancy measure between distributions.

\subsubsection{Score Matching}

The score matching (or the Fisher-{\hyv}) divergence, initially developed by \citet{scorematching1}, is the expected squared difference between the score functions for the two densities $\p(\mb{x})$ and $q(\mb{x})$: 
\begin{equation}
\mathcal{D}_{SM}(\p | q) = \Eqsmall{x}{q}\Big[\big\|\mb{h}(\mb{x})^{1/2} \odot (\score{q} - \score{p})\big\|^2 \Big],
\label{eq:smof1}
\end{equation}
where $\odot$ denotes element-wise multiplication. 
The inclusion of the weighting function $\mb{h}(\mb{x})$ yields the \textit{generalised score matching} divergence \citep{highdgraphical, pairwisegraphical, yu2019}, and $\mb{h}(\mb{x})=\mb{1}$ yields the classic score matching divergence. 

It can be seen that $\mathcal{D}_{SM}(\p | q)$ is simply the squared difference between two Langevin Stein operators $\mathcal{T}_{\p} h_l(\mb{x})$ and $\mathcal{T}_{q} h_l(\mb{x})$. Using \cref{eq:steinidentity}, $\mathcal{D}_{SM}(\p | q)$ can be rewritten as 
\[
\Eqsmall{x}{q}\Big[\sum^d_{l=1}h_l(\mb{x})({\scorel{p}}^2 + 2\dx \scorel{p}) + \dx h_l(\mb{x}) \scorel{p}\Big] + C_q,
\]
% Stein's identity \cref{eq:steinidentity} with the Langevin Stein operator, and $\mb{f}(\mb{x}) = \mb{h}(\mb{x}) \odot \score{p}$ can be rearranged as 
% \[
% \Eqsmall{x}{q}[\score{q} (\mb{h}(\mb{x}) \odot \score{p})] = - \Eqsmall{x}{q}[\mb\nabla_{\mb{x}}(\mb{h}(\mb{x}) \odot \score{p})]
% \]
% and then used to rewrite $\mathcal{D}_{SM}(\p | q)$ as follows
% \begin{align*}
%     &\Eqsmall{x}{q}\Big[\sum^d_{l=1}h_l(\mb{x})({\scorel{p}}^2-2\scorel{p}\scorel{q}+{\scorel{q}}^2)\Big] \\
%     =&\Eqsmall{x}{q}\Big[\sum^d_{l=1}h_l(\mb{x})({\scorel{p}}^2 + 2\dx \scorel{p}) + \dx h_l(\mb{x}) \scorel{p}\Big] + C_q, 
% \end{align*}
% as can be used to write the objective function in \cref{eq:smof1} as
% \begin{align}
% \Eqsmall{x}{q}\Big[\sum^d_{l=1}h_l(\mb{x})({\scorel{p}}^2& + 2\dx \scorel{p}) + \dx h_l(\mb{x}) \scorel{p}\Big] + C_q, \nonumber
% \end{align}
% where $\mb\nabla_{\mb{x}}\score{p}$ denotes the hessian of $\log \p(\mb{x})$, 
% and 
where $C_q = \Eqsmall{x}{q}[\score{q}\T\score{q}]$ can be considered a constant which can be safely ignored when minimising with respect to $\mb\theta$. 

% The use of Stein's identity again highlights the requirement of the boundary condition, \cref{eq:steinidentity_boundary}, in score matching.

% Therefore $\mathcal{D}_{SM}(\p | q)$ no longer depends on the unknown data PDF $q(\mb{x})$. The estimator of $\mb{\theta}$ is given by \(\mb{\hat{\theta}}\defeq \argmin_{\mb{\theta}} \mathcal{D}_{SM}(\p | q)\).

% By using Stein's identity, 
% The boundary condition of the density $q(\mb{x})$ is required for 
% this derivation
% \cref{eq:smof2} 
% to hold. 
% An alternate derivation of score matching can be given by directly applying integration by parts, which also requires the same conditions.

\subsection{Divergence for Densities with a Truncated Support} \label{sec:back:trunc}

Let $V \subset \R^d$ be a domain whose boundary is denoted by $\partial V$. 
The boundary condition of the density $q(\mb{x})$ given in \cref{eq:steinidentity_boundary} needs to hold on the boundary $\partial V$, i.e., the truncated density $q(\mb{x}') = 0, \forall \mb{x}' \in \partial V$. However, this is, in general, not true for truncated densities. For example, a 1D truncated unit Gaussian distribution within the interval $[-1, 1]$ has non-zero density at the boundary of the support at exactly $x=-1$ and $x=1$. When the boundary condition breaks down, the function families presented for classical SD, KSD and score matching are no longer Stein classes, and these divergences are no longer computationally tractable.

We outline two recent methods for circumventing this issue by modifying the KSD and the score matching divergence.

\subsubsection{Bounded-Domain KSD (bd-KSD)}\label{sec:bdksd}

Motivated by performing goodness-of-fit testing on truncated domains, \citet{bdksd} propose a modified Stein operator, given by
\begin{equation}
\mathcal{T}_{p, h}\mb{g}(\mb{x}) = \sum^d_{l=1} \scorel{p} g_l(\mb{x})h(\mb{x}) + \dx (g_l(\mb{x})h(\mb{x}))
\label{eq:bdksd_operator}
\end{equation}
for $\mb{g} \in \mathcal{G}^d$, where $h(\mb{x})$ is a weighting function for which $h(\mb{x}') = 0 \;\forall \mb{x}' \in \partial V$. 
This modified Stein operator relies on the  boundary conditions on $h$ instead of $q$. This condition can be satisfied by choosing $h$ carefully. 
% This modified Stein operator enforces the boundary condition on $h$, instead of relying on conditions on $q$ in \cref{eq:steinidentity_boundary}. 
%$h(\mb{x})\mb{g}(\mb{x})$ defines elements of a Stein class in the truncated domain.

% The bounded-domain Stein divergence is therefore defined by replacing the Stein operator in \cref{eq:sd} with this modified Stein operator. Taking the supremum over this new function family gives the bounded-domain kernelised Stein discrepancy (bd-KSD) divergence:
Using the Stein operator in \cref{eq:sd} and taking the supremum over $\mb{g} \in \mathcal{G}^d$ gives an alternative definition to KSD for bounded domains, called bounded-domain kernelised Stein discrepancy (bd-KSD):
\[
\mathcal{D}_{\text{bd-KSD}}(\p | q)^2 = \|\Eqsmall{x}{q} [\mathcal{T}_{\p, h} k(\mb{x}, \cdot)]\|_{\mathcal{G}^d}^2.
\]
The form of $h$ is not explicitly defined, but the authors recommend a distance function. For example, if $V$ is the unit ball, $h(\mb{x}) = 1 - \|\mb{x}\|^b$ for a chosen power $b$.

\subsubsection{Truncated Score Matching (TruncSM)} \label{sec:truncsm}

When the support of $q$ is the non-negative orthant, $\R^d_+$, \citet{scorematching2} and \citet{yu2019} impose restrictions on the weighting function $\mb{h}$ in the score matching divergence, given in \cref{eq:smof1}, such that $\mb{h}(\mb{x}) \to 0$ as $\|\mb{x}\| \to 0$, for example $\mb{h}(\mb{x}) = \mb{x}$. 
% This naturally satisfies the boundary criteria from \cref{eq:steinidentity_boundary} on $\mb{h}$ instead of $q$. 
\citet{yu2021} and \citet{song} generalised this constraint on $\mb{h}$ to any bounded domain $V$, imposing the restrictions $h_l(\mb{x}) > 0\; \forall l$ and $\mb{h}(\mb{x}') = \mb{0}$ for all $\mb{x}' \in \partial V$.

\citet{song} showed that maximising a Stein discrepancy with respect to $\mb{h}$ gives a solution
\(
\mb{h}_0 = (h_0, \dots, h_0),
\)
where
\begin{equation}
h_0 = \min_{\mb{x}' \in \partial V} \text{dist}(\mb{x}, \mb{x'}),
\label{eq:h0_tsm}
\end{equation}
the smallest distance between $\mb{x}$ and the boundary $\partial V$. \citet{song} proposed the use of several distance functions such as the $\ell_1$ and $\ell_2$ distance. The generalised score matching divergence with $h_0$ is referred to as \textit{TruncSM}.

\section{Approximate Stein Classes}

%\subsection{Motivation}

%The classical Stein identity given in \cref{def:steindiscrepancy} holds for any probability density function $p$ which is supported on $\R^d$. When support is truncated, e.g. $V \subset \R^d$, the boundary condition for which the density $p \to 0$ as $\mb{x} \in \R^d$ reaches the edges of the support ($\mb{x} \to \infty$ for $\mb{x} \in \R^d$) no longer holds, as the density is significantly non-zero for a given value of truncation. In this situation, we must rely on alternative conditions such that Stein's identity can hold. For example, we can require that $\mb{f}(\mb{x}) = 0$ when $\mb{x}$ is at the boundary of our support, and then Stein's identity is once again satisfied. 

%If we choose no closed form solution to $\mb{f}$, and we instead desire that it is a data-driven function which is optimised with respect to our dataset, this condition must be held for arbitrary parameters $\mb\theta$ and dataset $\bm{x}$. It becomes clear that the condition $\mb{f}(\mb{x}) = 0$ at the boundary of the domain can only be satisifed to a degree of accuracy which depends on the accuracy of the constraint being optimised over the dataset. This in turn means that Stein's identity also holds within a tolerance of accuracy depending on the constrained optimisation.

%\subsection{Proposal}

So far, previous works have assumed a functional form of the truncation boundary, so that $h$ can be precisely designed and Stein's identity holds exactly. 
% meaning Stein's identity holds precisely. 
In our setting, the functional form of the truncation boundary is unavailable, making the design of a Stein class infeasible; we cannot design a class of functions with appropriate boundary conditions that satisfy Stein's identity exactly. 
However, the truncation boundary is known to us approximately as a set of boundary points, and so we can design a set of functions such that Stein's identity holds `approximately', as we will show below. 

Let us first define an approximate Stein class and an approximate Stein identity. 

\begin{definition}\label{def:approxstein}
    Let $q = q(\mb{x})$ be a smooth probability density function on $V \subseteq \R^d$. $\widetilde{\mathcal{F}}_m^d$ is an \textit{approximate Stein class} of $q$, if for all $\mb{\tilde{f}} \in \widetilde{\mathcal{F}}_m^d$,  
    \begin{equation}
        \Eqsmall{x}{q} [ \mathcal{S}_{q, m} \mb{\tilde{f}} (\mb{x})] = O_P(\varepsilon_m),
        \label{eq:approxstein}
    \end{equation}
    where $\varepsilon_m$ is a monotonically decreasing function of $m$, and $\mathcal{S}_{q, m}$ is a Stein operator that depends on $m$ in some way.
    $O_P(\varepsilon_m)$ denotes a sequence indexed by $m$ which is bounded in probability by $\varepsilon_m$.
    % \citep{asymptoticstatistics}.
\end{definition}

We denote \cref{eq:approxstein} as the approximate Stein identity, and the approximate Stein class $\widetilde{\mathcal{F}}_m^d$ can be written more simply as 
\begin{align}
\label{eq.approx.stein.class}
\widetilde{\mathcal{F}}_m^d = \{\mb{\tilde{f}}: V \to \R \;|\; \Eqsmall{x}{q} [ \mathcal{S}_{q, m} \mb{\tilde{f}} (\mb{x})] = O_P(\varepsilon_m)\}.
\end{align}
The classical Stein class of functions given by \cref{def:steinidentity} requires Stein's identity to hold, whereas this approximate Stein class \cref{eq.approx.stein.class} defines a set of functions for which 
$\Eqsmall{x}{q} [ \mathcal{S}_{q,m} \mb{\tilde{f}} (\mb{x})]$
is bounded by a decreasing sequence. Similarly to the classical Stein discrepancy, we propose the use of the Langenvin Stein operator, i.e. $\mathcal{S}_{q, m} = \mathcal{T}_{q, m}$. 
{\color{blue} We aim to show this is a flexible class which can be used across various applications, of which we provide two examples below.}
\st{Below we give two examples of approximate Stein classes for different scenarios.}

\textbf{Example: Latent Variable Models.}

\citet{kanagawa2019} presented a variation of KSD for testing the goodness-of-fit of models with unobserved latent variables $\mb{z}$, where the density of $\mb{x}$ is given by
\(
q(\mb{x}) = \int_{\R^d} q(\mb{x} | \mb{z}) q(\mb{z}) d\mb{z},
\)
and thus the score function can be written as
\begin{align}
\score{q}(\mb{x}) = \frac{\mb\nabla_{\mb{x}} q(\mb{x})}{q(\mb{x})} &= \int_{\R^d}\frac{\mb\nabla_{\mb{x}} q(\mb{x}|\mb{z})}{q(\mb{x} | \mb{z})} \cdot \frac{q(\mb{x}|\mb{z})q(\mb{z})}{q(\mb{x})}d\mb{z} \nonumber \\
&= \Eqsmall{z}{q(\mb{z} | \mb{x})} [\score{q}(\mb{x} | \mb{z})].
\label{eq:latent_score}
\end{align}
{\color{blue} We show that this existing modification of KSD gives rise to an Approximate Stein Class.}
To evaluate the expectation over $q(\mb{z} | \mb{x})$, \citet{kanagawa2019} recommend approximating the expectation with a Monte Carlo estimate 
\begin{equation}
\Eqsmall{z}{q(\mb{z}|\mb{x})} [\score{q} (\mb{x} | \mb{z})] {\color{blue} = \frac{1}{m}\sum^m_{i=1} \score{q} (\mb{x} | \mb{z}_{i}) + O_P(\varepsilon_m)},
\label{eq:latent_monte_carlo}
\end{equation}
where we assume we have access to $m$ unbiased samples $\{\mb{z}_{i}\}^m_{i=1} \sim q(\mb{z}|\mb{x})${\color{blue} , and $O_P(\varepsilon_m)$ is the Monte Carlo approximation error}.
{\color{blue}
% When \cref{eq:latent_monte_carlo} 
This approximation leads to  
% is substituted into \cref{eq:steinidentity}, we find that 
\[
\Eqsmall{x}{q} [ \mathcal{T}_{q} \mb{g}(\mb{x})] = O_P(\varepsilon_m) \neq 0
\]
and therefore this variation of KSD uses an Approximate Stein Class.
}
\st{ Under some mild conditions, this latent variable score function, when applying Monte Carlo estimation, gives rise to an approximate Stein identity.} See \cref{app:latent} for {\color{blue} more} details.

\section{KSD for Truncated Density Estimation} \label{sec:tksd}

Let $q(\mb{x})$ be a smooth probability density function with truncated support $V \subset \R^d$ with boundary $\partial V$. When $q(\mb{x})$ has a truncated support, $\mathcal{G}^d$ is not a Stein class. To show this, let $\mb{g} \in \mathcal{G}^d$, then Stein's identity can be written as
\begin{align}
    \Eqsmall\brs{\mathcal{T}_{q} \mb{g}(\mb{x})} &= \int_{V}  q(\mb{x}) \Big({\sum^d_{l=1}\psi_{\p, l}(\mb{x}) g_l(\mb{x}) + \dx g_l(\mb{x})}\Big)d \mb{x}\nonumber \\
    &= \sum_{l=1}^d  {\int_V \partial_{x_l} q(\mb{x}) g_l(\mb{x}) + q(\mb{x})\partial_{x_l} g_l(\mb{x}) d\mb{x}}   \nonumber \\
    &= \oint_{\partial V} q(\mb{x})\sum^d_{l=1}  g_l(\mb{x}) \hat{\mathrm{u}}_l(\mb{x})ds,\label{eq:tksd:approx_stein_truncated}
\end{align}
where $\oint_{\partial V}$ is the surface integral over the boundary $\partial V$, $\hat{\mathrm{u}}_1(\mb{x}), \dots, \hat{\mathrm{u}}_d(\mb{x})$ are the unit outward normal vectors on $\partial V$ and $ds$ is the surface element on $\partial V$.

To satisfy the Stein identity, \cref{eq:tksd:approx_stein_truncated} must equal zero. In the untruncated setting, \cref{eq:tksd:approx_stein_truncated} is zero due to the boundary condition on $q(\mb{x})$, but in the truncated setting the density is significantly nonzero on the boundary at $\partial V$. Prior methods such as \citet{bdksd} choose a pre-defined weighting function for which \cref{eq:tksd:approx_stein_truncated} is exactly zero at the boundary. We instead aim to show that \cref{eq:tksd:approx_stein_truncated} is $O_P(\varepsilon_m)$ under an approximate Stein class, detailed in the remainder of this section.

\subsection{Approximate Stein Class for Truncated Densities}
Let us first consider a setting where the boundary $\partial V$ is known analytically. Define a modified product RKHS as
\begin{equation}
\mathcal{G}^d_0 = \{ \mb{g} \in \mathcal{G}^d \;|\; \mb{g}(\mb{x}') = \mb{0} \: \forall \mb{x}' \in \partial V, \|\mb{g}\|^2_{\mathcal{G}^d} \leq 1 \}.
\label{eq:G0}
\end{equation}

\begin{lemma}\label{lem:steinidentity}
Let $q$ be a smooth density supported on $V$. For any $\mb{g} \in \mathcal{G}^d_0$, then
\begin{equation}
\Eqsmall{x}{q} [ \mathcal{T}_{q} \mb{g}(\mb{x}) ] = 0.
\label{eq:tksdsteinidentity}
\end{equation}
\end{lemma}
\st{For proof,} {\color{blue} Similar to the classic Stein's identity, the proof follows from applying a simple integration by parts,} see \Cref{app:lemSI}.  \Cref{lem:steinidentity} shows that $\mathcal{G}^d_0$ is a proper Stein class of $q$. $\mathcal{G}^d_0$ defines a large class of functions, for which the proposed operator by \citet{bdksd} given in \cref{eq:bdksd_operator} is one such example of a function from this family. 

Let us now consider the setting where $\partial V$ is not known exactly. The information about the boundary is provided by $\approxdV = \{\mb{x}'_{i'}\}_{i'=1}^m$, a finite set of points randomly sampled from $\partial V$. Define
\begin{equation}
\approxG = \{ \mb{g} \in \mathcal{G}^d \;|\; \mb{g}(\mb{x}') = \mb{0} \: \forall \mb{x}' \in \approxdV, \|\mb{g}\|^2_{\mathcal{G}^d} \leq 1 \},
\label{eq:G02}
\end{equation}
which can be considered as an approximate version of $\mathcal{G}^d_0$ using the finite set $\approxdV$. 
% In practice, sampling uniform points from $\partial V$ is straightforward. For example, if $V$ is the unit ball around the origin, then uniform points can be sampled by normalising a unit Multivariate Gaussian distribution centred on the origin. 
The benefit of using $\approxG$ over $\mathcal{G}^d_0$ 
% will become clear in the next section.
is that this class can be constructed using only `partial' boundary information, i.e., $\approxdV$, without knowing the explicit expression of $\partial V$.

First, we show specific properties of the relationship between $\approxdV$ and $\partial V$.

%We aim to show that $\approxG$ is an approximate Stein class of $q$, for which we need the following Lemma.

% Using this, we can prove the following result.
\begin{lemma}\label{lem:gtilde_to_g}
    Let $\mb{g} \in \mathcal{G}^d_0$ and $\tilde{\mb{g}} \in \approxG$. Assume that $\approxdV$ is $\varepsilon_m$-dense in $\partial V$ with some probability. Further assume that $g_l$ and $\tilde{g}_l$ are $C$-Lipschitz continuous for all $l = 1, \dots, d$. Then 
    % with high probability, $| g_l(\mb{x}') - \tilde{g}_l(\mb{x}') | \leq C \varepsilon_m$, and
    \[
        | g_l(\mb{x}') - \tilde{g}_l(\mb{x}') | = O_P(\varepsilon_m)
    \]
    for any $\mb{x}' \in \partial V$. 
\end{lemma}
\st{For proof} {\color{blue} The proof follows from applying the Lipschitz continuous property on $g_l$ and $\tilde{g}_l$ and then the triangle inequality,} see \Cref{app:gtilde_to_g}. \Cref{lem:gtilde_to_g} establishes a connection between $\mathcal{G}^d_0$ and $\approxG$, relying on the assumption that $\approxdV$ is $\varepsilon_m$-dense in $\partial V$. We now show under some mild conditions 
this assumption 
holds with high probability.

\begin{proposition}\label{lem:dense}
    %Let $p$ be the probability that $\approxdV$ is $\varepsilon_m$ dense in $\partial V$. 
    Assume $\{\mb x_i' \}_{i=1}^m$ are samples drawn from the uniform distribution defined on $\partial V$.  
    Let $\mathrm{L}(V)$ denote the $(d-1)$-surface area of a bounded domain $V \subset \mathbb{R}^d$ and $\mathrm{L}(V) < \infty$. 
     Let $B_{\varepsilon_m}(\mb{x}')$ denote a ball of radius $\varepsilon_m$ centred on $\mb{x}'$, and let $\xi(d) = \pi^{d/2}/\Gamma(\frac{d}{2}+1)$.
    % Assume that the surface area of $\partial V \cap B_{\varepsilon_m}(\mb{x}') \leq \xi\varepsilon_m$, where $\xi(d) > 0$ is some constant. 
    For all $\varepsilon_m$ such that 
    \begin{equation}
		\varepsilon_m \geq \br{\frac{\mathrm{L}(V)}{\xi(d)}\brs{1 - 0.05^{1/m}}}^{1/d}.
        \label{eq:epsilon}
    \end{equation}
    we have
    \[
        \mathbb{P}\br{\approxdV \text{ is }  \varepsilon_m\text{-dense} \text{ in } \partial V} = 0.95.
    \]
\end{proposition}
For proof, see \cref{app:dense}. \Cref{eq:epsilon} shows the relationship between $m$ and $\varepsilon_m$, which corresponds to how `close' $\approxdV$ is to $\partial V$. 
\begin{remark}\label{cor:epsilon_decreasing}
Our numerical investigation (\cref{app:epsilon_decreasing}) shows that the bound on $\varepsilon_m$, as defined in \cref{eq:epsilon}, is a decreasing function of $m$, and is not sensitive to the value of $\mathrm{L}(V)$.
\end{remark}
% One can see this result by inspection of \cref{eq:epsilon}, but 
% we have provided further empirical justification in \cref{app:epsilon_decreasing}.
% \Cref{lem:dense} shows that with high probability, $\approxdV$ is $\varepsilon_m$-dense in $\partial V$, and \cref{cor:epsilon_decreasing} shows that $\varepsilon_m$ is a decreasing function of $m$. 

\Cref{lem:dense} shows that \cref{lem:gtilde_to_g} holds with high probability, which in turn enables us to show that indeed $\approxG$ is an approximate Stein class. 
\begin{theorem}\label{thm:steinidentityGtilde}
Assume the conditions specified in \cref{lem:gtilde_to_g} hold. Let $q$ be a smooth density supported on $V$. For any $\tilde{\mb{g}} \in \approxG$, then
\begin{equation}
\Eqsmall{x}{q} [ \mathcal{T}_{q} \tilde{\mb{g}}(\mb{x}) ] = O_P(\varepsilon_m).
\label{eq:steinidentityGtilde}
\end{equation}
\end{theorem}
\st{For proof} {\color{blue} The proof again follows from integration by parts, then applying \cref{lem:gtilde_to_g}}, see \Cref{app:steinidentityGtilde}. 
% \Cref{thm:steinidentityGtilde}.
% shows that $\tilde{\mathcal{G}}^d_0$ is an approximate Stein class of $q$ according to \cref{def:approxstein}. 
This result paves the way for designing a new type of Stein divergence that measures differences between two distributions when their domain $V$ is truncated.

\subsection{Truncated Kernelised Stein Discrepancy (TKSD) and Density Estimation}

Let $q$ and $\p$ be two smooth densities supported on $V$. We can construct a Stein discrepancy measure,
called the truncated Kernelised Stein Discrepancy (TKSD), given by
\begin{equation}
    \mathcal{D}_{\text{TKSD}}(\p | q) := \sup_{\mb{g} \in \approxG} \Eqsmall{x}{q}[\mathcal{T}_{\p} \mb{g}(\mb{x})].
    \label{eq:tksd1}
\end{equation}
% As \cref{thm:steinidentityGtilde} stated, 
Similarly to classical SD and KSD, TKSD can still be intuitively thought as the maximum violation of Stein's identity with respect to
% its corresponding Stein identity, based on 
an \emph{approximate} Stein class $\approxG$. 
% However, TKSD can still be computed with only random samples of the boundary. In contrast, earlier KSD measure for truncated density, bd-KSD requires the functional form of the boundary. 
It can be used to distinguish two distributions when their domain $V$ is truncated, but the boundary is not known analytically. 

$\mathcal{D}_{\text{TKSD}}(\p | q)$ serves as the discrepancy measure between densities. Later, we propose to estimate a truncated density function by minimising this discrepancy. 

Next, we show that \st{$\mathcal{D}_{\text{TKSD}}(\p | q)^2$ has a closed form expression.} {\color{blue} there is an analytic solution to the constrained optimisation problem in \cref{eq:tksd1}.}
% solution to the supremum given in \cref{eq:tksd1}.
\begin{theorem}\label{thm:ksd}
$\mathcal{D}_{\text{TKSD}}(\p | q)^2$ can be written as
\begin{equation}
\sum^d_{l=1} \Eq{x}{q} \Eq{y}{q} \brs{u_l(\mb{x}, \mb{y}) - \bm{v}_l(\mb{x})\T (\bm{K}')^{-1}\bm{v}_l(\mb{y})}
\label{eq:tksd}
\end{equation}
where $u_l(\mb{x}, \mb{y})$ is given by \cref{eq:ul}, $\bm{v}_l(\mb{z}) =\scorel{p}(\mb{z}) \kxxp{z}\T + (\dz  \kxxp{z})\T$,  $ \kxxp{z} = [ k(\mb{z}, \mb{x}_1'), \dots, k(\mb{z}, \mb{x}_m') ]$, $\mb\phi_{\bm{x}'} = [k(\mb{x}_1', \cdot), \dots, k(\mb{x}_m', \cdot)]\T$ and $\bm{K}' = \mb\phi_{\bm{x}'}\mb\phi_{\bm{x}'}\T$.
\end{theorem}
\st{For proof} {\color{blue} The proof relies on solving a Lagrangian dual problem}, see \Cref{app:theoremksd}. 
This result {\color{blue} gives a closed form loss function for $\mathcal{D}_{\text{TKSD}}(\p | q)^2$, and} is not straightforward to obtain since the constraints in $\approxG$ are enforced on a finite set of points in $\mathbb{R}^d$\st{, thus \cref{eq:tksd1} is a constrained problem}. {\color{blue} \Cref{eq:tksd} can be decomposed into the `KSD part', given by the $u_l(\mb{x}, \mb{y})$ term, and the `truncated part', given by $\bm{v}_l(\mb{x})\T (\bm{K}')^{-1}\bm{v}_l(\mb{y})$, which comes from solving for the Lagrangian dual parameter. \Cref{eq:tksd} is also linear in $d$, so its evaluation cost only increases linearly with $d$.}

Next, we show that the TKSD can be approximated with samples from $q$. 
% has an unbiased estimator via samples from $q$.
\begin{theorem}
    Let $\{\mb{x}_i\}^n_{i=1}$ be a set of samples from $q(\mb{x})$, then  
    \begin{equation}
    \widehat{\mathcal{D}}_{\text{TKSD}}(\p | q)^2 = \frac{2}{n(n-1)}\sum^n_{i=1}\sum^n_{\substack{j=1\\i\neq j}} h(\mb{x}_i, \mb{x}_j),
    \label{eq:Ustatistic}
    \end{equation}
    is an unbiased estimate of $\mathcal{D}_{\text{TKSD}}(\p | q)^2$, where 
    \begin{equation}
h(\mb{x}_i, \mb{x}_j) = \sum^d_{l=1} u_l(\mb{x}_i, \mb{x}_j) - \bm{v}_l(\mb{x}_i)\T (\bm{K}')^{-1}\bm{v}_l(\mb{x}_j),
    \label{eq:h-ustatistic}
    \end{equation}
    assuming that $\Eq{x}{q}\Eq{y}{q}[h(\mb{x}, \mb{y})^2] \leq \infty$.
\end{theorem}
The proof follows directly from the definition of a $U$-statistic \citep{serfling2009}. Additionally, we could define a biased estimate of $\mathcal{D}_{\text{TKSD}}(\p | q)^2$ via a $V$-statistic
\begin{equation}
\widehat{\mathcal{D}}_{\text{TKSD}}(\p | q)^2 = \frac{1}{n^2}\sum^n_{i=1}\sum^n_{j=1} h(\mb{x}_i, \mb{x}_j).
\label{eq:Vstatistic}
\end{equation}
The $U$-statistic and $V$-statistic for TKSD allow us to evaluate $\mathcal{D}_{\text{TKSD}}(\p | q)^2$ via $n$ samples from $q$.
In empirical experiments, the $V$-statistic seems to give better performance overall compared to the $U$-statistic. 

Finally, we define our proposed estimator for unnormalised truncated density by minimising TKSD over the density parameter $\mb \theta$. 
\begin{equation}
\label{eq.estimator}
    \hat{\mb\theta}_{n, m} := \argmin_{\mb\theta} \widehat{\mathcal{D}}_{\text{TKSD}}(\p | q)^2.
\end{equation}

In the next section, we study the theoretical properties of \cref{eq.estimator}. 

\subsection{Consistency Analysis}
Since $\widehat{\mathcal{D}}_{\text{TKSD}}(\p | q)^2$ is a function of $m$, 
for the simplicity of the theorem, let us study \cref{eq.estimator} at the limit of $m$. 
Let $\hat{L} (\mb \theta) := \lim_{m \to \infty} \widehat{\mathcal{D}}_{\text{TKSD}} (\p | q)^2$ and define 
\begin{align*}
    \hat{\mb \theta}_n := \argmin_{\mb \theta} \hat{L}(\mb \theta).
\end{align*}
We now prove that $\hat{\mb \theta}_n$ converges to the true parameter under mild conditions:

\begin{assumption} [Accurate Boundary Prediction]
\label{ass.least.square}
Let 
\begin{align*}
    &\mb t := \E_{\mb{y}}\brs{(\kxxp{y}\T[\mb\nabla_{\mb\theta}\scorel{p}(\mb{y}))\T]\big|_{\mb\theta = \mb\theta^\star}} \in \mathbb{R}^{m \times \mathrm{dim}(\mb \theta)}, \\
    &\mb t(\mb x) := \E_{\mb{y}}\brs{( \mb{\varphi}^\top_{\mb{y}, \mb{x}} [\mb\nabla_{\mb\theta}\scorel{p}(\mb{y}))\T]\big|_{\mb\theta = \mb\theta^\star}} \in \mathbb{R}^{\mathrm{dim}(\mb \theta)}. 
\end{align*}
Assume the following holds:
\begin{equation}
\begin{split}
& \left[\oint_{\partial V} q(\mb{x})\kxxp{x} (\bm{K}')^{-1} \mb t \hat{\mathrm{u}}_l(\mb{x}) ds \right]_i\\
=& \left[\oint_{\partial V} q(\mb{x}) \mb t(\mb x) \hat{\mathrm{u}}_l(\mb{x}) ds \right]_i + O_P(\hat\varepsilon_{m}), \\
\forall i& \in \{1, \dots, \mathrm{dim}(\mb \theta)\}, l \in \{1, \dots, d\}
\end{split}
\label{eq:approx1}
\end{equation}
\end{assumption}
where $\hat\varepsilon_{m}$ is a positive decaying sequence with respect to $m$ and $\lim_{m\to \infty} \hat\varepsilon_{m} = 0$. 

One can see that $\kxxp{x} (\bm{K}')^{-1} \mb t$ is the kernel least-square regression prediction of $\mb t(\mb x')$, for a $\mb x' \in \partial V$.  \cref{ass.least.square} essentially states that the least squares `trained' on our boundary samples, $\{\mb x'_{i'}\}_{i'=1}^m$, should be asymptotically accurate in terms of a testing error computed over a surface integral as $m$ increases. 

\begin{assumption}
\label{ass.eigenvalue}
    The smallest eigenvalue of the Hessian of $\hat{L} (\mb \theta)$  is lower bounded, i.e., 
    $\lambda_\mathrm{min}\left[ \nabla^2_{\mb \theta} \hat{L}(\mb \theta) \right] \ge \Lambda_\mathrm{{min}} > 0$ with high probability. 
\end{assumption}
In fact, we could make the same assumption on the population version of the objective function, and the convergence of the sample hessian matrix would guarantee Assumption \ref{ass.eigenvalue} holds with high probability. To simplify our theoretical statement we stick to the simpler version. 
% $L(\mb \theta) := \mathbb{E}_q \left[ \lim_{m \to \infty} \widehat{\mathcal{D}}_{\text{TKSD}} (\mb \theta) \right] = \lim_{m \to \infty} \mathbb{E}_q \left[ \widehat{\mathcal{D}}_{\text{TKSD}} (\mb \theta) \right]$. 

\begin{theorem}\label{thm:theta_star}
% \todo{finish this proof and introduce it/state its relevance}
% Let $\hat{\mb\theta}_{n, m}$ be the minimiser of \cref{eq:tksd} for a fixed $n$ and $m$. 
Suppose there exists a unique $\mb\theta^\star$ such that $q = p_{\mb\theta^\star}$, \cref{ass.least.square}  and \cref{ass.eigenvalue} hold, 
Then
\[
\|\hat{\mb\theta}_{n} - \mb\theta^\star\| = O_P\br{\frac{1}{\sqrt{n}}}. 
\]
\end{theorem}
For proof, see \cref{app:theta_star}.
This result shows, although our TKSD is an approximate version of KSD, the density estimator derived from TKSD is still a consistent estimator. {\color{blue} We empirically verify this consistency in \cref{app:consistency}.}

\section{Experimental Results}
To show the validity of our proposed method, we experiment on benchmark settings against \textit{TruncSM} and an adaptation of bd-KSD for truncated density estimation. We also provide additional empirical experiments in the appendices; empirical consistency (\cref{app:consistency}), a demonstration on the Gaussian mixture distribution (\cref{app:mixture}), an implementation for truncated regression (\cref{app:regression}) and a investigation into the effect of the distribution of the boundary points (\cref{app:boundary_dist}).
%The bd-KSD method was originally proposed for goodness-of-fit testing, and our initial experiments show that using it for truncated density estimation yields a similar result to \textit{TruncSM}, with added computational complexity. For simplicity, in these experiments we consider \textit{TruncSM} as `state-of-the-art', and use only it for comparison purposes.

\subsection{Computational Considerations}

% The $U$-statistic and $V$-statistic, given in \cref{eq:Ustatistic} and \cref{eq:Vstatistic}, can be coded in a fully vectorised fashion.
% %, of which formulae are given in \cref{app:vectorised}. 
% When minimising with respect to $\mb\theta$, certain terms of the objective, such as the second derivative of the kernel function, can be treated as a constant and thus ignored, increasing efficiency further.

TKSD requires the selection of hyperparameters: the number of boundary points, $m$, the choice of kernel function, $k$, and the corresponding kernel hyperparameters. We focus on the Gaussian kernel, $k(\mb{x}, \mb{y}) = \exp\{-(2\sigma^2)^{-1}\|\mb{x} - \mb{y}\|^2\}$, and the bandwidth parameter $\sigma$ is chosen heuristically as the median of pairwise distances on the data matrix. 

In choosing $m$, there is a trade-off between accuracy and computational expense, since $\bm{K}'$ in \cref{eq:h-ustatistic} is an $m\times m$ matrix which requires inversion. In experiments, we let $m$ scale with $d^2$. We provide more computational details of the method in \cref{app:extracomp}.

When the boundary's functional form is unknown, the recommended distance functions by \citet{bdksd} and \citet{song} cannot be used, and instead \textit{TruncSM} and bd-KSD must use approximate boundary points. This approximation to the distance function is given by
\begin{equation}
\min_{\mb{x}' \in \approxdV} \|\mb{x} - \mb{x}'\|_\alpha^\gamma,
\label{eq:projections}
\end{equation}
for each dataset point $\mb{x}$, where $\gamma$ and $\alpha$ are chosen based on the application. This may not provide an accurate approximation to the true distance function when $m$ is small.

\subsection{Density Truncated by the Boundary of the United States} \label{sec:usa}

\begin{figure}[t!]
    \centering
    \includegraphics[width=\linewidth]{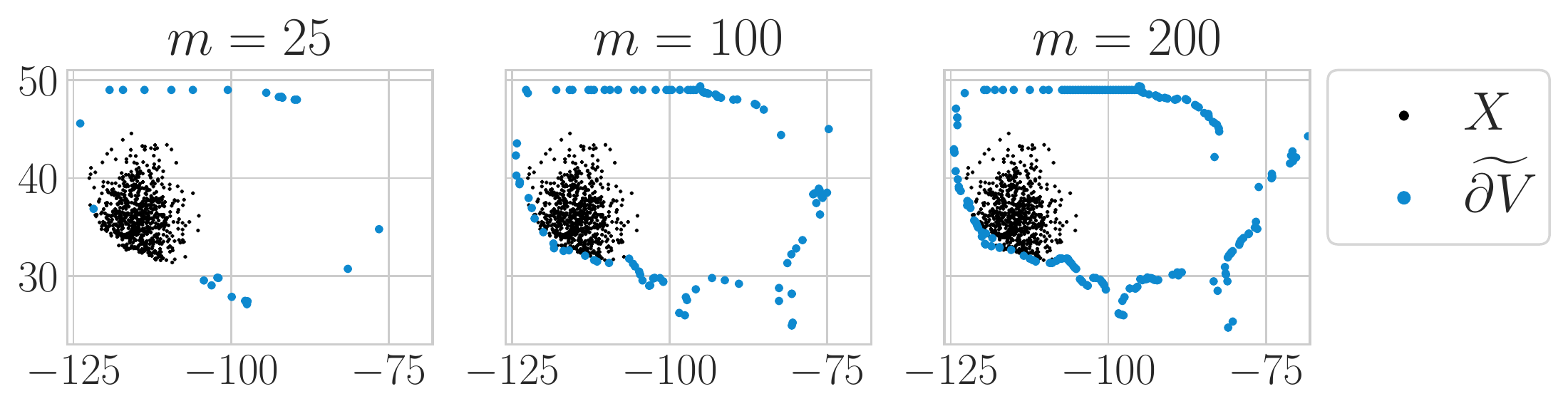}
    \includegraphics[width=\linewidth]{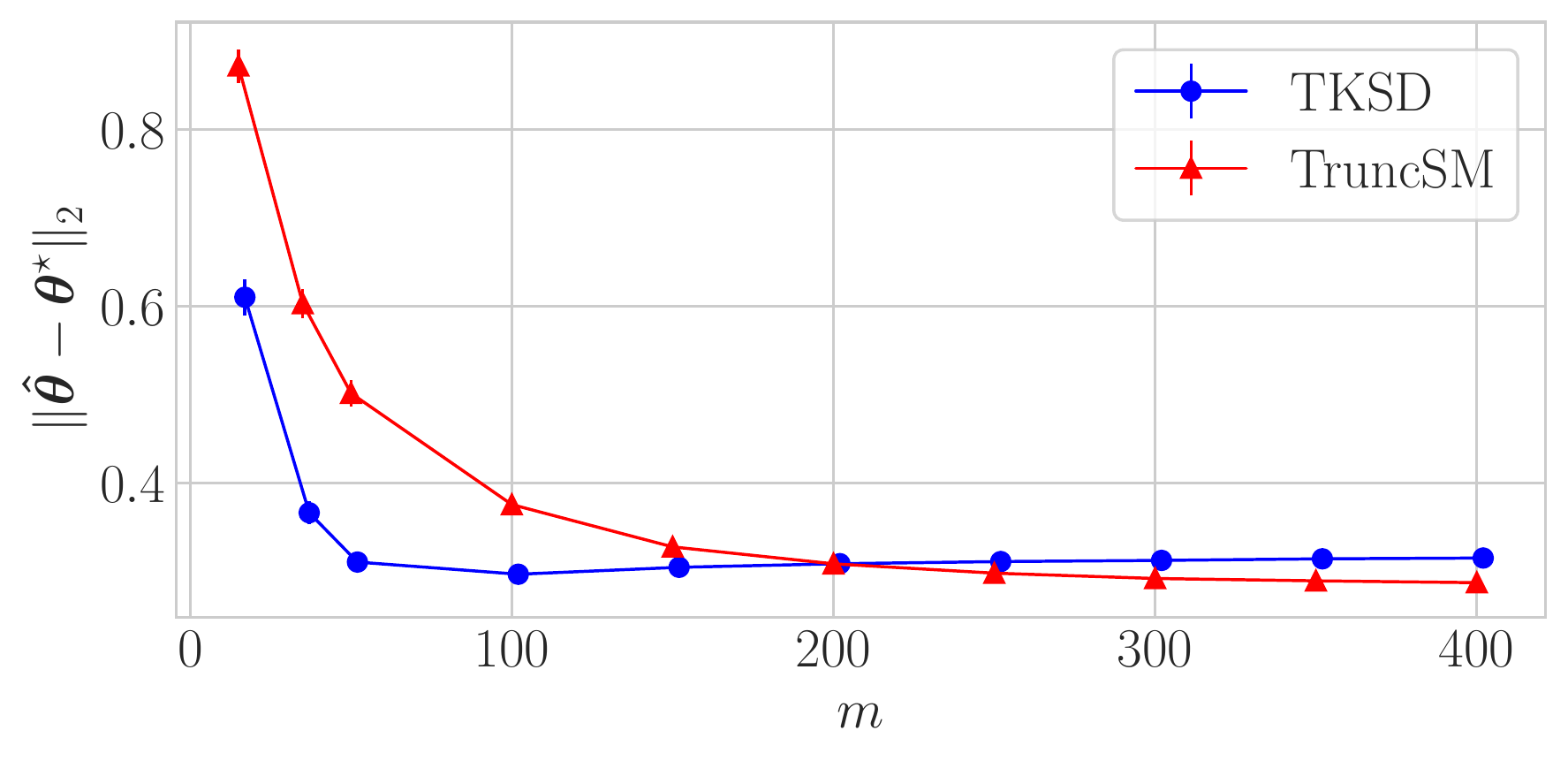}
    \caption{Density estimation when the truncation boundary is the border of the U.S., as described in \cref{sec:usa}. Top: example of increasing the number of boundary points $m$. Bottom: across 256 seeds for each value of $m$, mean estimation error with standard error bars for the mean of a 2D Gaussian, for TKSD and \textit{TruncSM} as $m$ increases.}
    \label{fig:usaexample}
\end{figure}

Let us consider the complicated boundary of the United States (U.S.). Let $V$ be the interior of the U.S. and let $\approxdV$ be a set of coordinates (longitude and latitude) which define the country's borders. The boundary of the U.S. is a highly irregular shape, and as such, there is no explicit expression of the boundary in this case. \textit{TruncSM} and bd-KSD must use an approximate distance function (given by \cref{eq:projections}), whereas TKSD can readily use the set of coordinates that give rise to the boundary. 

The experiment is set up as follows. Let $\mb\mu^{\star} = [-115, 35]$ and $\mb\Sigma = 10\cdot I_d$. Samples are simulated from $\mathcal{N}(\mb\mu^{\star}, \mb\Sigma)$ and we select only those which are in the interior of the U.S. until we reach $n=400$ points. \st{We estimate the mean $\mb{\hat{\mu}}$ of the truncated dataset, assuming $\mb\Sigma$ is known, using the TKSD objective function, and also compare it to a fit using the \textit{TruncSM} objective only.} {\color{blue} Assuming $\mb\Sigma$ is known, we estimate $\mb\mu^{\star}$ with $\mb{\hat{\mu}}$ using TKSD and compare it to the estimation using \textit{TruncSM}.} We also vary \st{the number of boundary points ($m$)}{\color{blue} $m$} by uniformly sampling from the perimeter of the U.S., demonstrated in \cref{fig:usaexample} (top).

Figure \ref{fig:usaexample} (bottom) shows the mean and standard error of the $\ell_2$ estimation error between $\mb\mu^{\star}$ and $\mb{\hat{\mu}}$, measured over 256 trials for each value of $m$. \textit{TruncSM} improves with higher values of $m$ as the approximate distance function increases in accuracy, whilst TKSD performs significantly better with fewer boundary points.

\subsection{Estimation Error and Dimensionality}\label{sec:dimensionbench}

\begin{figure}[t!]
    \centering
    \includegraphics[width=\linewidth]{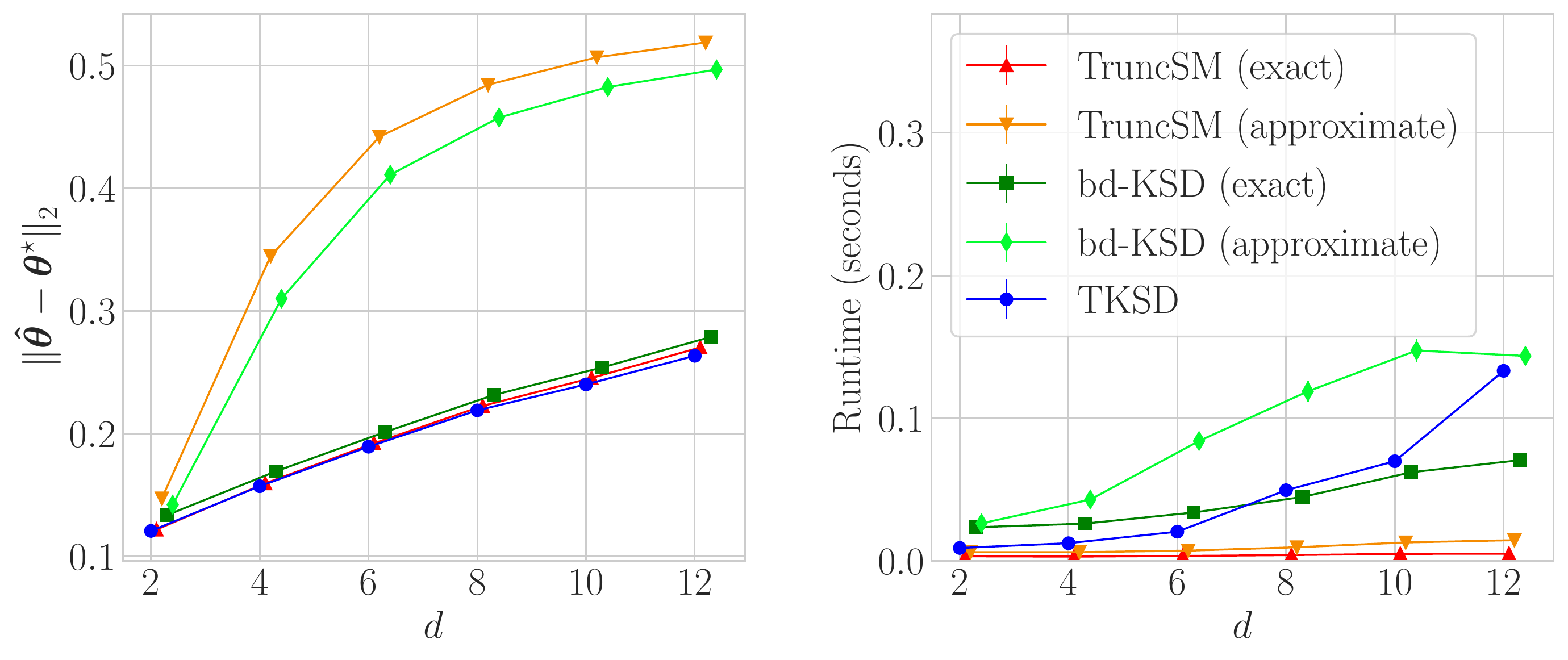}
    \includegraphics[width=\linewidth]{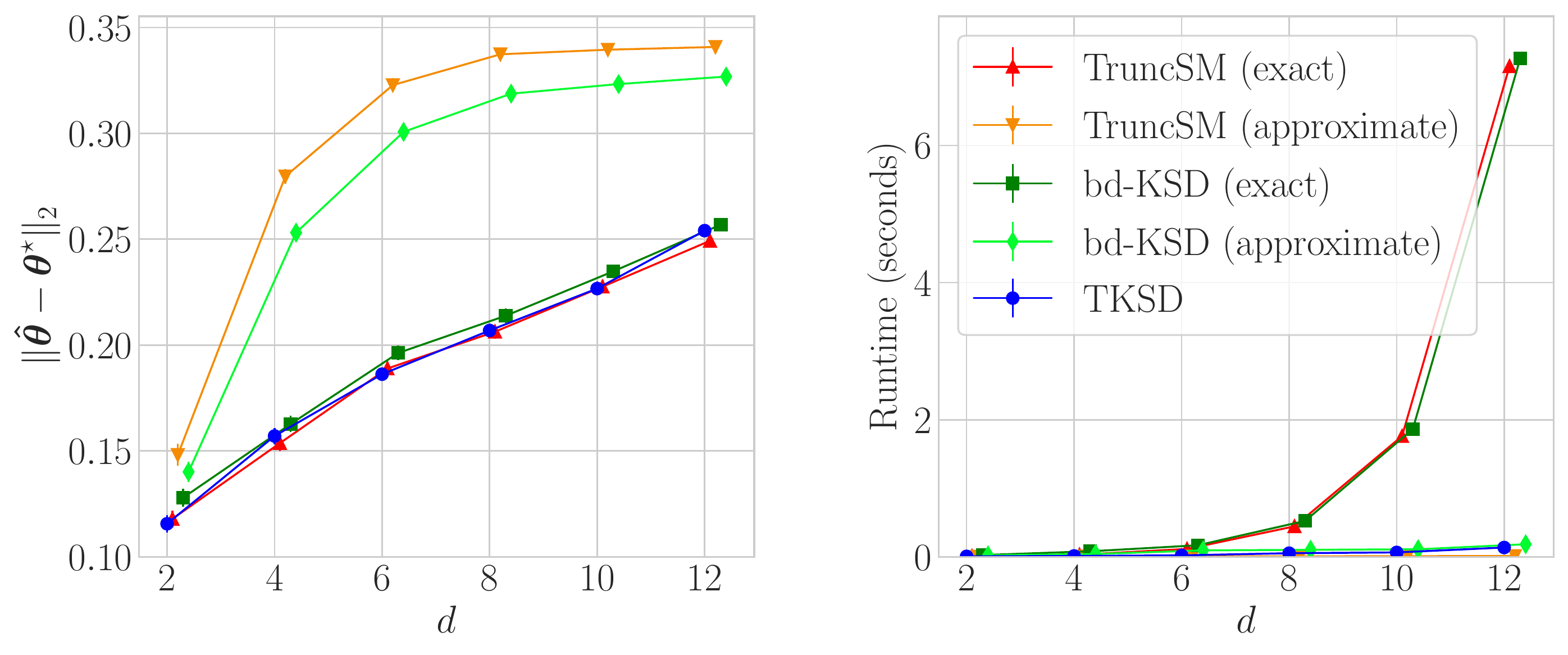}
    \caption{Mean estimation error across 256 seeds, with standard error bars, as dimension $d$ increases (left) and runtime for each method (right). The truncation domain is the $\ell_2$ ball of radius $d^{0.53}$ (top) and $\ell_1$ ball of radius $d$ (bottom).}
    \label{fig:dimension_bench}
\end{figure}

Consider a simple experiment setup where $\mb\mu_d^{\star}= \mb{1}_d \cdot 0.5$, samples are simulated from $\mathcal{N}(\mb\mu_d^{\star}, \mb{I}_d)$ and are truncated be within the $\ell_c$ ball for $c=1$ and $c=2$, each of radius $r$, until we reach $n=300$ data points. We choose radii $r=d^{0.53}$ for $c=1$ and $r=d^{0.98}$ for $c=2$, chosen so that the amount of points truncated across dimension is roughly 50\% (see \cref{app:bsize} for details). \st{In this experiment, we are interested in estimating $\mb\mu$, and we measure the estimation error of $\mb{\hat{\mu}}$ against the true $\mb\mu^{\star}$ as dimension $d$ increases, measured by the $\ell_2$ norm, $\|\mb{\hat{\mu}} - \mb\mu^{\star}\|_2$. Each boundary point $\mb{x}' \in \approxdV$, is simulated from $\mb{x}' \sim \mathcal{N}(\mb{0}_d, \mb{I}_d)$, then normalised by dividing by $\|\mb{x}'\|_c$ and multiplying by $r$, giving a random set of points on the boundary.}
{\color{blue}
We estimate $\mb\mu_d^\star$ with $\mb{\hat{\mu}}$, and measure the $\ell_2$ estimation error against $\mb\mu_d^\star$ as $d$ increases. Each boundary point $\mb{x}' \in \approxdV$ is simulated from $\mathcal{N}(\mb{0}_d, \mb{I}_d)$, normalized by $\|\mb{x}'\|_c$, and multiplied by $r$, resulting in a set of random points on the boundary.
}

\Cref{fig:dimension_bench} shows the error and computation time for the following estimators: 
\begin{itemize}
    \itemsep0em
    \item TKSD: our method as described in \cref{sec:tksd}, using a randomly sampled $\approxdV$.
    \item \textit{TruncSM}/bd-KSD (exact): the implementation by \citet{song}/\citet{bdksd} respectively, where the distance function is computed exactly using the known boundaries.
    \item \textit{TruncSM}/bd-KSD (approximate): the implementation by \citet{song}/\citet{bdksd} respectively\st{, where the distance function is approximated by solving \cref{eq:projections}} {\color{blue} with distance function given by \cref{eq:projections}}, using the same $\approxdV$ as given to TKSD.
\end{itemize}

For the $\ell_2$ case, TKSD marginally outperforms all competitors across all dimensions, at only a slight computational expense. 
For the $\ell_1$ case, TKSD, bd-KSD and \textit{TruncSM} have similar estimation errors across all dimensions. However, \textit{TruncSM} (exact) and bd-KSD (exact) have an increasing computation time due to the costly evaluation of the distance function.
In this implementation, we follow the advice of \citet{song}, Section 7, where the distance function to the $\ell_1$ ball is calculated via 
a closed-form expression, for which the computational complexity increases combinatorically with dimension.

\textit{TruncSM} (approximate) and bd-KSD (approximate) have significantly higher estimation error than other methods across all benchmarks. TKSD is able to achieve the same level of accuracy as the exact methods, using the same finite set of boundary points, $\approxdV$.

\section{Discussion}

We have proposed an alternative to the classical Stein class, called an \textit{approximate Stein class}, bounded by a decreasing sequence instead of being strictly equal to zero. By maximising the KSD objective over this approximate Stein class, we have constructed a truncated density estimator based on KSD, called truncated KSD (TKSD), and shown that it is consistent. TKSD has advantages over the prior works by \citet{bdksd} and \citet{song}, as it does not require a functional form of a boundary.

Some limitations of this method include the requirement of selecting hyperparameters, such as the kernel function $k$ and its associated hyperparameters, and the number of samples from the boundary, $m$. {\color{blue} Choice of $m$ may depend on applications. Still, a larger value is preferred when the complexity of the truncation boundary is higher, or the dimension increases.} However, even with the heuristic approaches presented in this paper, TKSD provides competitive results. 

In experimental results, we have shown that even though we assume no access to a functional form of the boundary, TKSD performs similarly to previous methods which require the boundary to be computed. In some scenarios, TKSD performs better than these methods, or takes less time to achieve the same result.

\section*{Reproducibility}
All results in this paper can be reproduced using the GitHub repository located at \url{https://github.com/dannyjameswilliams/tksd}.

\section*{Acknowledgements}
We thank all four reviewers for their insightful feedback and suggestions to improve the paper. We are particularly grateful for their recommendations of additional experiments. We would also like to thank Jake Spiteri, Jack Simons, Michael Whitehouse, Mingxuan Yi and Dom Owens, for their helpful input throughout the development of this work.

Daniel J. Williams was supported by a PhD studentship from the EPSRC Centre for Doctoral Training in Computational Statistics and Data Science.

\bibliography{main}
\bibliographystyle{icml2023_cr/icml2023}

%%%%%%%%%%%%%%%%%%%%%%%%%%%%%%%%%%%%%%%%%%%%%%%%%%%%%%%%%%%%%%%%%%%%%%%%%%%%%%%
%%%%%%%%%%%%%%%%%%%%%%%%%%%%%%%%%%%%%%%%%%%%%%%%%%%%%%%%%%%%%%%%%%%%%%%%%%%%%%%
% APPENDIX
%%%%%%%%%%%%%%%%%%%%%%%%%%%%%%%%%%%%%%%%%%%%%%%%%%%%%%%%%%%%%%%%%%%%%%%%%%%%%%%
%%%%%%%%%%%%%%%%%%%%%%%%%%%%%%%%%%%%%%%%%%%%%%%%%%%%%%%%%%%%%%%%%%%%%%%%%%%%%%%
\newpage
\appendix
\onecolumn

\section{Proofs and Additional Theoretical Results}

\subsection{Latent Variable Approximate Stein Identity} \label{app:latent}
Firstly, we show that the score function, $\mb\psi_q(\mb{x})$, is given by
\begin{align}
\score{q}(\mb{x}) = \frac{\nabla_{\mb{x}} q(\mb{x})}{q(\mb{x})} &= \int_{\R^d} \frac{\nabla_{\mb{x}} (q(\mb{x} | \mb{z}) q(\mb{z}))}{q(\mb{x})}\frac{q(\mb{x}|\mb{z})}{q(\mb{x}|\mb{z})} d\mb{z} \nonumber \\
&= \int_{\R^d} \frac{q(\mb{x}|\mb{z})q(\mb{z})}{q(\mb{x})} \brs{\frac{\nabla_{\mb{x}} q(\mb{x} | \mb{z})}{q(\mb{x} | \mb{z})}}  d\mb{z} \nonumber \\
&= \int_{\R^d} q(\mb{z}|\mb{x}) \nabla_{\mb{x}} \log q(\mb{x}|\mb{z})  d\mb{z} \nonumber\\
&= \E_{\mb{z} \sim q(\mb{z} | \mb{x})} [\latentscorexz],
\label{eq:app:latent_score}
\end{align}
where $\latentscorexz = \nabla_{\mb{x}} \log q(\mb{x} | \mb{z})$. To evaluate the expectation over $q(\mb{z} | \mb{x})$, \citet{kanagawa2019} recommend approximating each element with a Monte Carlo estimate 
\begin{equation}
\E_{\mb{z} \sim  q(\mb{z} | \mb{x})} [\latentscorexzj]  = \frac{1}{m}\sum^m_{i=1} \latentscorexzji + O_P(\varepsilon_m),
\label{eq:app:latent_monte_carlo}
\end{equation}
% add a note and define O_P as elementwise boundedness, is interpreted in an elementwise fashion
% each element in the vector is stochastically bounded by epsilon_m
where $\latentscorexzj = \partial_{x_j} \log q(\mb{x} | \mb{z})$ and we assume we have access to $m$ samples $\{\mb{z}_{i}\}^m_{i=1} \sim q(\mb{z}|\mb{x})$, and $O_P(\varepsilon_m)$ is the Monte Carlo approximation error which decreases as the number of samples on the boundary, $m$, increases. 

The Stein identity with this score function is written asWe show that this existing modification of Stein discrepancies give rise to an Approximate Stein Class. Suppose $\mb{f}\in \mathcal{F}^d$, where $\mathcal{F}^d$ is a regular Stein class. Stein's identity with the score function from \cref{eq:app:latent_score} can be written as
\begin{align*}
\Eqsmall{x}{q}[\mathcal{T}_q \mb{f}(\mb{x})] &= \Eqsmall{x}{q}\brs{\sum^d_{j=1} \psi_{q,j}(\mb{x}) f_j(\mb{x}) + \partial_{x_j} f_j(\mb{x})} \\
&= \Eqsmall{x}{q}\brs{\E_{\mb{z} \sim q(\mb{z} | \mb{x})} \brs{\latentscorexzj}f_j(\mb{x}) + \partial_{x_j} f_j(\mb{x})}.
\end{align*}
Substituting \eqref{eq:app:latent_monte_carlo} into the above gives
\begin{align*}
    \Eqsmall{x}{q}[\mathcal{T}_q \mb{f}(\mb{x})] &=  \Eqsmall{x}{q}\brs{\sum^d_{j=1} \brc{ \br{\frac{1}{m}\sum^m_{i=1}\latentscorexzji + O_P(\varepsilon_m)}f_j(\mb{x}) + \partial_{x_j} f_j(\mb{x})}} \\ 
    &= \Eqsmall{x}{q}\brs{\sum^d_{j=1} \brc{ \frac{1}{m}\sum^m_{i=1}\latentscorexzji f_j(\mb{x}) + \partial_{x_j} f_j(\mb{x}) + O_P(\varepsilon_m) f_j(\mb{x})}} \\
    &\myeq{(a)} \Eqsmall{x}{q}\brs{\sum^d_{j=1}\brc{ \frac{1}{m}\sum^m_{i=1}\latentscorexzji f_j(\mb{x}) + \partial_{x_j} f_j(\mb{x})}} + O_P(\varepsilon_m) \\
    &\myeq{(b)}  O_P(\varepsilon_m),
\end{align*}
where equality (a) follows from the fact that the error in the Monte Carlo approximation is accounted for by the $O_P(\varepsilon_m)$ term, and equality (b) follows by definition of $\mathcal{F}^d$ being a Stein class. Therefore, $\Eqsmall{x}{q}[\mathcal{T}_q \mb{f}(\mb{x})] = O_P(\varepsilon_m)$, and this latent variable modification uses an Approximate Stein Class.

\subsection{Proof of \Cref{lem:steinidentity}}\label{app:lemSI}
Let $\mb{g} \in \mathcal{G}^d_0$. Begin by writing the expectation as an integral,
\begin{align*}
    \Eqsmall{x}{q} [ \mathcal{T}_q \mb{g}(\mb{x}) ] = \int_V q(\mb{x}) \mathcal{T}_q \mb{g}(\mb{x}) d\mb{x} &= \sum^d_{l=1} \int_V q(\mb{x}) \dx \log q(\mb{x}) g_l(\mb{x})  + \dx g_l(\mb{x}) d\mb{x} \\
 &= \sum^d_{l=1} \int_V \dx q(\mb{x}) g_l(\mb{x})  + \dx g_l(\mb{x}) d\mb{x}.
\end{align*}
This can be expanded by integration by parts,
\begin{align*}
    \sum^d_{l=1} \int_V \dx q(\mb{x}) g_l(\mb{x})  + \dx g_l(\mb{x}) d\mb{x} &= \sum^d_{l=1} \oint_{\partial V} q(\mb{x}) g_l(\mb{x}) \hat{\mathrm{u}}_l(\mb{x}) ds + \sum^d_{l=1} \int_V q(\mb{x}) (\dx g_l(\mb{x})  - \dx g_l(\mb{x}))d\mb{x}  \\
    &= \sum^d_{l=1} \oint_{\partial V} q(\mb{x}) g_l(\mb{x}) \hat{\mathrm{u}}_l(\mb{x}) ds = 0,
\end{align*}
where the final equality comes from all evaluations $g_l(\mb{x}') = 0\; \forall \mb{x}' \in \partial V$.

\subsection{Proof of \Cref{lem:gtilde_to_g}.}\label{app:gtilde_to_g}

Let $\mb{g} \in \mathcal{G}^d_0$ and $\tilde{\mb{g}} \in \approxG$. First note that $\mb{g}$ and $\tilde{\mb{g}}$ agree on $\approxdV$, i.e.
\begin{equation}
g_l(\mb{\tilde{x}}') = \tilde{g}_l(\mb{\tilde{x}}'), \;\;\forall l=1,\dots, d
\label{eq:gagree}
\end{equation}
for all $\mb{\tilde{x}}' \in \approxdV$, since all $\mb{\tilde{x}}'$ are elements of $\partial V$ also, as $\approxdV\subset \partial V$.
First, we note that since $g_l$ and $\tilde{g}_l$ are both Lipschitz continuous, then
\begin{align}
    |g_l(\mb{x}') - g_l(\mb{\tilde{x}}')| &\leq C_1\|\mb{x}' - \mb{\tilde{x}}'\| \leq C_1\varepsilon_m \label{eq:glipdense1}\\
    |\tilde{g}_l(\mb{\tilde{x}}') - \tilde{g}_l(\mb{x}')| &\leq C_2\|\mb{x}' - \mb{\tilde{x}}'\| \leq C_2\varepsilon_m \label{eq:glipdense2},
\end{align}
where $\|\mb{x}' - \mb{\tilde{x}}'\| \leq \varepsilon_m$. This follows under the assumption that $\approxdV$ is $\varepsilon_m$-dense in $\partial V$, therefore the distance $\|\mb{x}' - \mb{\tilde{x}}'\|$ is at most $\varepsilon_m$. 

We seek to quantify
\[
| g_l(\mb{x}') - \tilde{g}_l(\mb{x}') |,
\]
which is how far apart the `approximate' $\tilde{g}_l$ is from the `true' $g_l$ for any point $\mb{x}' \in \partial V$. Note that this includes points that are not in $\approxdV$. We let
\[
g_l(\mb{x}') - \tilde{g}_l(\mb{x}') = (g_l(\mb{x}') - \tilde{g}_l(\mb{x}')) + (g_l(\mb{\tilde{x}}') - g_l(\mb{\tilde{x}}')) + (\tilde{g}_l(\mb{\tilde{x}}') - \tilde{g}_l(\mb{\tilde{x}}')),
\]
and by \eqref{eq:gagree}, 
\begin{align*}
g_l(\mb{x}') - \tilde{g}_l(\mb{x}') &= (g_l(\mb{x}') - \tilde{g}_l(\mb{x}')) - g_l(\mb{\tilde{x}}') + \tilde{g}_l(\mb{\tilde{x}}') \\
&= (g_l(\mb{x}') - g_l(\mb{\tilde{x}}')) + (\tilde{g}_l(\mb{\tilde{x}}') - \tilde{g}_l(\mb{x}')).
\end{align*}
Using \eqref{eq:glipdense1} and \eqref{eq:glipdense2}, we have the following inequality, 
\[
| g_l(\mb{x}') - \tilde{g}_l(\mb{x}') | \leq |g_l(\mb{x}') - g_l(\mb{\tilde{x}}')| + |\tilde{g}_l(\mb{\tilde{x}}') - \tilde{g}_l(\mb{x}')| \leq (C_1 + C_2)\varepsilon_m,
\]
where the last step follows by the triangle inequality. Therefore, we have
\[
| g_l(\mb{x}') - \tilde{g}_l(\mb{x}') | = O_P(\varepsilon_m)
\]
as desired.

\subsection{Proof of \Cref{lem:dense}}\label{app:dense}

%We show that for any fixed $\varepsilon_m$, the probability that $\approxdV$ is $\varepsilon_m$-dense in $\partial V$ tends to 1 as $m \to \infty$. 
First, let $\mathrm{L}(V)$ denote the $(d-1)$-surface area of a bounded domain $V \subset \mathbb{R}^d$ and $\mathrm{L}(V) < \infty$. For example, in 2D, $\mathrm{L}(V)$ corresponds to the line length of $\partial V$. We also define $B_{\varepsilon_m}(\mb{x}')$ as the ball of radius $\varepsilon_m$ centred on $\mb{x}'$.

Before continuing to the proof, recall the definition of an $\varepsilon$-dense set: $\forall \mb{x}' \in \partial V, \exists\hspace{0.5mm} \widetilde{\mb{x}}' \in \approxdV$ such that $d(\mb{x}', \widetilde{\mb{x}}') \leq \varepsilon$, where $d$ is some measure of distance. The statement $d(\mb{x}', \widetilde{\mb{x}}') \leq \varepsilon_m$ is equivalent to $\widetilde{\mb{x}}' \in B_{\varepsilon_m}(\mb{x}')$. Now write that the probability that the definition of a $\varepsilon_m$-dense set holds for $\approxdV$ being dense in $\partial V$ is equal to 0.95:
\begin{align}
0.95 = \mathbb{P}(\forall \mb{x}' \in \partial V, \exists\hspace{0.5mm} \widetilde{\mb{x}}' \in \approxdV,  \widetilde{\mb{x}}' \in B_{\varepsilon_m}(\mb{x}')) &= 1 - \mathbb{P}(\exists\hspace{0.5mm} \mb{x}' \in \partial V, \forall \widetilde{\mb{x}}' \in \approxdV, \widetilde{\mb{x}}' \notin B_{\varepsilon_m}(\mb{x}')) \nonumber \\
&= 1 - \mathbb{P}\Big(\bigcup_{\mb{x}' \in \partial V} \forall \widetilde{\mb{x}}' \in \approxdV, \widetilde{\mb{x}}' \notin B_{\varepsilon_m}(\mb{x}')\Big). \nonumber 
\end{align}
Equivalently, by rearranging the above, we have
\[
\mathbb{P}\Big(\bigcup_{\mb{x}' \in \partial V} \forall \widetilde{\mb{x}}' \in \approxdV, \widetilde{\mb{x}}' \notin B_{\varepsilon_m}(\mb{x}')\Big) = 0.05.
\]
This probability can be bounded as follows
\[
0.05 = \mathbb{P}\Big(\bigcup_{\mb{x}' \in \partial V} \forall \widetilde{\mb{x}}' \in \approxdV, \widetilde{\mb{x}}' \notin B_{\varepsilon_m}(\mb{x}')\Big) \geq \mathbb{P}\Big(\forall \widetilde{\mb{x}}' \in \approxdV, \widetilde{\mb{x}}' \notin B_{\varepsilon_m}(\mb{x}'_0)\Big)
\]
where ${\mb{x}}_0'$ is a single point in $\partial V$. This follows by basic laws of probability.
Note that the probability on the right hand side can be written as 
\begin{align}
    \mathbb{P}\br{\forall \widetilde{\mb{x}}' \in \approxdV, \widetilde{\mb{x}}' \notin B_{\varepsilon_m}(\mb{x}'_0)} &= \mathbb{P}\br{\widetilde{\mb{x}}'_0 \notin B_{\varepsilon_m}(\mb{x}'_0)}^m \nonumber\\
    &= \brs{1 - \mathbb{P}\br{\widetilde{\mb{x}}'_0 \in B_{\varepsilon_m}(\mb{x}'_0)}}^m \nonumber\\
    &= \brs{1 - \frac{\mathrm{Area}(\partial V \cap B_{\varepsilon_m}(\mb{x}'_0))}{\mathrm{L}(V)}}^m \leq 0.05 \label{eq:tksd:propdense_1}
\end{align}
where the first equality comes from all $\widetilde{\mb{x}}' \in \approxdV$ being independently distributed. $\mathrm{Area}(S)$ denotes the surface area of the boundary set $S$. For example, $\mathrm{Area}(\partial V) = \mathrm{L}(V)$. Therefore $\mathrm{Area}(\partial V \cap B_{\varepsilon_m}(\mb{x}'_0))$ represents the size of the region of $\partial V$ that is inside $B_{\varepsilon_m}(\mb{x}'_0)$, which will be the area of the boundary hyperplane of $\partial V$ which passes through $B_{\varepsilon_m}(\mb{x}'_0)$. We obtain the probability in the final equality by assuming that all $\widetilde{\mb{x}}'$ are uniformly sampled from $\partial V$, so the probability is the proportion of $\partial V$ inside $B_{\varepsilon_m}(\mb{x}'_0)$ as a ratio of the full $\partial V$. This probability exists under the assumption that $\mathrm{L}(V) < \infty$. 

Next, we can rearrange \cref{eq:tksd:propdense_1} to
\[
\mathrm{Area}(\partial V \cap B_{\varepsilon_m}(\mb{x}'_0)) \geq \mathrm{L}(V)\brs{1 - 0.05^{1/m}}. 
\]
Consider an upper bound for $\mathrm{Area}(\partial V \cap B_{\varepsilon_m}(\mb{x}'_0))$ given by the volume of $B_{\varepsilon_m}(\mb{x}'_0)$, i.e.
\[
\mathrm{Area}(\partial V \cap B_{\varepsilon_m}(\mb{x}'_0)) \leq \xi(d) \varepsilon_m^d, \qquad \xi(d) = \frac{\pi^{d/2}}{\Gamma(\frac{d}{2}+1)},
\]
where $\Gamma$ is the Gamma function. Combining these bounds gives
\[
\xi(d) \varepsilon_m^d \geq \mathrm{L}(V)\brs{1 - 0.05^{1/m}}
\]
and therefore
\[
\varepsilon_m \geq \br{\frac{\mathrm{L}(V)}{\xi(d)}\brs{1 - 0.05^{1/m}}}^{1/d},
\]
which completes the proof.

\begin{comment}
\textbf{Aside: $\approxdV$ is dense in $\partial V$.}

Let $\varepsilon_0$ be the value of $\varepsilon_m$ such that $|\partial V \cap B_{\varepsilon_0}(\mb{x})| = |\partial V|$, i.e. $B_{\varepsilon_0}(\mb{x})$ is sufficiently large so that it `encompasses' $\partial V$. Then for any $\varepsilon_m \geq \varepsilon_0$, $|\partial V \cap B_{\varepsilon_m}(\mb{x})| = |\partial V|$ always, no matter how large $\varepsilon_m$ grows. Thus for any $\varepsilon_m \geq \varepsilon_0$, $|\partial V \cap B_{\varepsilon_m}(\mb{x})|/|\partial V| = 1$, hence $1 - (1 - {|\partial V \cap B_{\varepsilon_m}(\mb{x})|}/{|\partial V|})^m = 1$ independently of $m$.

Consider $\varepsilon_m < \varepsilon_0$. Then $|\partial V \cap B_{\varepsilon_m}(\mb{x})| < |\partial V|$ and  $|\partial V \cap B_{\varepsilon_m}(\mb{x})|/|\partial V| < 1$. Hence $(1 - {|\partial V \cap B_{\varepsilon_m}(\mb{x})|}/{|\partial V|})^m \to 0$ as $m$ increases, and therefore $1-(1 - {|\partial V \cap B_{\varepsilon_m}(\mb{x})|}/{|\partial V|})^m \to 1$ as $m \to \infty$. 

Since the probability is 1 for $\varepsilon_m \geq \varepsilon_0$, and tends to 1 for $\varepsilon_m < \varepsilon_0$ as $m \to \infty$, then $\approxdV$ is $\varepsilon_m$-dense in $\partial V$ for all $\varepsilon_m$ in the case that $m \to \infty$. Hence in this case $\approxdV$ is dense in $\partial V$.
\end{comment}

\subsection{Proof of \Cref{thm:steinidentityGtilde}} \label{app:steinidentityGtilde}
Let $\mb{g} \in \mathcal{G}^d_0$ and $\tilde{\mb{g}} \in \approxG$. Begin with writing the expectation in its integral form,
\begin{align}
\Eqsmall{x}{q} [ \mathcal{T}_{q} \tilde{\mb{g}}(\mb{x}) ] &= \int_V q(\mb{x}) \brs{\sum_{l=1}^d  \dx \log q(\mb{x}) \tilde{g}_l(\mb{x}) + \sum_{l=1}^d \dx \tilde{g}_l(\mb{x})} d\mb{x} \nonumber\\
&=\sum_{l=1}^d \int_V q(\mb{x}) \dx \log q(\mb{x}) \tilde{g}_l(\mb{x})d\mb{x} + \sum_{l=1}^d \int_V q(\mb{x}) \dx \tilde{g}_l(\mb{x}) d\mb{x} \nonumber \\
&\myeq{(a)}\sum_{l=1}^d \int_V \dx q(\mb{x}) \tilde{g}_l(\mb{x}) d\mb{x}  + \sum_{l=1}^d \int_V q(\mb{x}) \dx \tilde{g}_l(\mb{x}) d\mb{x} \nonumber \\
&\myeq{(b)}\sum_{l=1}^d \oint_{\partial V} q(\mb{x}) \tilde{g}_l(\mb{x})\hat{\mathrm{u}}_l(\mb{x}) ds - \sum_{l=1}^d \int_V q(\mb{x}) \dx \tilde{g}_l(\mb{x}) d\mb{x}  + \sum_{l=1}^d \int_V q(\mb{x}) \dx \tilde{g}_l(\mb{x}) d\mb{x} \nonumber  \\
&=\sum_{l=1}^d \oint_{\partial V} q(\mb{x}) \tilde{g}_l(\mb{x})\hat{\mathrm{u}}_l(\mb{x}) ds \nonumber
\end{align}
where $\oint_{\partial V}$ is the surface integral over the boundary $\partial V$, (a) comes from the identity that $q(\mb{x}) \dx \log q(\mb{x}) = \dx q(\mb{x})$, and (b) is from integration by parts. Substitute $\tilde{g}_l(\mb{x}) = \tilde{g}_l(\mb{x}) + g_l(\mb{x}) - g_l(\mb{x})$ to obtain
\begin{align*}
&\sum_{l=1}^d \oint_{\partial V} q(\mb{x}) (\tilde{g}_l(\mb{x}) + g_l(\mb{x}) - g_l(\mb{x}))\hat{\mathrm{u}}_l(\mb{x}) ds \\
=&\sum_{l=1}^d \oint_{\partial V} q(\mb{x}) (\tilde{g}_l(\mb{x}) - g_l(\mb{x}))\hat{\mathrm{u}}_l(\mb{x}) ds + \sum_{l=1}^d\oint_{\partial V} q(\mb{x}) g_l(\mb{x}) \hat{\mathrm{u}}_l(\mb{x}) ds.
\end{align*}

By \cref{lem:gtilde_to_g}, $| \tilde{g}_l (\mb{x}') - g_l(\mb{x}') | = O_P(\varepsilon_m)$, leaving
\[
\Eqsmall{x}{q} [ \mathcal{T}_{q} \tilde{\mb{g}}(\mb{x}) ] = O_P(\varepsilon_m) + \sum_{l=1}^d \oint_{\partial V} q(\mb{x}) g_l(\mb{x}) \hat{\mathrm{u}}_l(\mb{x}) ds
\]
As all evaluations of $g_l(\mb{x}) = 0\; \forall \mb{x} \in \partial V$ by definition of $\mathcal{G}^d_0$, the second term equals zero, leaving only 
\[
\Eqsmall{x}{q} [ \mathcal{T}_{q} \tilde{\mb{g}}(\mb{x}) ] = O_P(\varepsilon_m)
\]
as desired.

\subsection{Proof of \Cref{thm:ksd}}\label{app:theoremksd}

Recall $\approxG = \{ \mb{g} \in \mathcal{G}^d \;|\; \mb{g}(\mb{x}') = \mb{0} \: \forall \mb{x}' \in \approxdV, \|\mb{g}\|^2_{\mathcal{G}^d} \leq 1 \}$, where $\mathcal{G}^d$ is the product RKHS with $d$ elements, $\mb{g} = (g_1, \dots, g_d)$, and $g_l \in \mathcal{G}, \forall l=1,\dots, d$. 

Begin with the definition of TKSD as defined in \cref{eq:tksd1},
\[
\sup_{\mb{g} \in \approxG} \Eqsmall{x}{q}[\mathcal{T}_{\p} \mb{g}(\mb{x})].
\]
To solve this supremum analytically, we can reframe it as a constrained maximisation problem,
\begin{align}
    \max_{\mb{g} \in \mathcal{G}^d}\: &\Eqsmall{x}{q}[\mathcal{T}_{\p} \mb{g}(\mb{x})], \label{eq:constr_opt}\\
    \text{subject to } & g_l(\mb{x}') = 0\; \forall \mb{x}' \in \approxdV, \; \forall l = 1, \dots, d \nonumber \\
                       & \|\mb{g}\|_{\mathcal{G}^d} \leq 1,\nonumber
\end{align}
where the constraints in the definition for $\approxG$ have been included as optimisation constraints, and the maximisation is now with respect to the RKHS function family $\mathcal{G}^d$ only. To solve this, we can formulate a Lagrangian dual function \citep{convex}.
\begin{align}
     \inf_{\mb{\nu}^{(1)}, \dots, \mb\nu^{(d)}, \lambda}&\mathcal{L}(\mb\theta, \mb{g}, \mb\nu^{(1)}, \dots, \mb\nu^{(d)}, \lambda)\\
     \mathcal{L} = \mathcal{L}(\mb\theta, \mb{g}, \mb\nu^{(1)}, \dots, \mb\nu^{(d)}, \lambda) := \;&\Eqsmall{x}{q}[\mathcal{T}_{\p} \mb{g}(\mb{x})] + \sum^d_{l=1} \sum^m_{i'=1} \nu^{(l)}_{i'} g_l (\mb{x}'_{i'}) + \lambda (\|\mb{g}\|^2_{\mathcal{G}^d} - 1),
    \label{eq:dual} 
\end{align}
for $\mb\nu^{(l)} \in \R^m \; \forall l$, $\lambda \geq 0$ and $\mathcal{L}$ is our Lagrangian. The overall optimisation problem that needs solving is given by
\begin{equation}
    \min_{\mb{\nu}^{(1)}, \dots, \mb\nu^{(d)}, \lambda} \;\;\max_{\mb{g}\in\mathcal{G}^d} \; \mathcal{L}(\mb\theta, \mb{g}, \mb\nu^{(1)}, \dots, \mb\nu^{(d)}, \lambda).
    \label{eq:minmax}
\end{equation}
By solving the dual problem, \cref{eq:minmax}, we solve the primal problem, \cref{eq:constr_opt}.
We can rewrite \eqref{eq:dual} as 
\[
\mathcal{L} = \Eqsmall{x}{q}\brs{\sum^d_{l=1}\dx \log \p(\mb{x})g_l(\mb{x}) + \dx g_l(\mb{x})}+ \sum^d_{l=1}\sum^m_{i'=1} \nu^{(l)}_{i'} g_l (\mb{x}'_{i'}) + \lambda (\|\mb{g}\|^2_{\mathcal{G}^d} - 1),
\]
and expand evaluations of $g_l(\mb{x})$ via the reproducing property of $\mathcal{G}$, given by $g_l(\mb{x}) = \langle g_l, k(\mb{x}, \cdot)\rangle_{\mathcal{G}}$, to
\begin{align}
\mathcal{L} &= \Eqsmall{x}{q}\brs{\sum^d_{l=1}\dx \log \p(\mb{x})\brangle{g_l, k(\mb{x}, \cdot)}_{\mathcal{G}} + \brangle{g_l, \dx k(\mb{x}, \cdot)}_{\mathcal{G}}} + \sum^m_{i'=1} \sum^d_{i=1}\nu^{(l)}_{i'} \brangle{g_l, k(\mb{x}'_{i'}, \cdot)}_{\mathcal{G}} + \lambda \br{\sum^d_{l=1}\langle g_l, g_l\rangle_{\mathcal{G}} - 1} 
\nonumber \\
&= \sum^d_{l=1}\Eqsmall{x}{q}\brs{\brangle{g_l, \dx \log \p(\mb{x})k(\mb{x}, \cdot) + \dx k(\mb{x}, \cdot) + \sum^m_{i'=1} \nu^{(l)}_{i'} k(\mb{x}'_{i'}, \cdot)}_{\mathcal{G}}} + \lambda \br{\sum^d_{l=1}\langle g_l, g_l\rangle_{\mathcal{G}} - 1}
\nonumber \\
&=\sum^d_{l=1}\Big\langle g_l, \Eqsmall{x}{q}[\dx \log \p(\mb{x})k(\mb{x}, \cdot) + \dx k(\mb{x}, \cdot)] + \sum^m_{i'=1} \nu^{(l)}_{i'} k(\mb{x}'_{i'}, \cdot)\Big\rangle_{\mathcal{G}} + \lambda \br{\sum^d_{l=1}\langle g_l, g_l\rangle_{\mathcal{G}} - 1}. \label{eq:dual2}
\end{align}

The final equality holds provided that $\Eqsmall{x}{q}[\dx \log \p(\mb{x})k(\mb{x}, \cdot) + \dx k(\mb{x}, \cdot) + \sum^m_{i'=1} \nu^{(l)}_{i'} k(\mb{x}'_{i'}, \cdot)] < \infty$, i.e. the term inside the expectation is Bochner integrable \citep{steinwart2008support}. This same assumption was made in \citet{kernelgoodness}, as we can consider $\sum^m_{i'=1} \nu^{(l)}_{i'} k(\mb{x}'_{i'}, \cdot)$ as a constant with respect to the expectation. 

We solve for each parameter via differentiation to obtain a closed form solution. Across dimensions, each $g_l$ in \eqref{eq:dual2} appears only additively to another, and so we can consider the $l$-th element of the derivative, and solve the inner maximisation for each $g_l$ to give a solution for $\mb{g}$. This differentiation gives
\[
    \frac{\partial \mathcal{L}}{\partial g_l} = \Eqsmall{x}{q}\brs{\dx \log \p(\mb{x})k(\mb{x}, \cdot) + \dx k(\mb{x}, \cdot)} + \sum^m_{i'=1} \nu^{(l)}_{i'} k(\mb{x}'_{i'}, \cdot) + 2\lambda g_l = 0.
\]
Rearranging for $g_l$ gives the solution
\[
g_l^{\star} = -\frac{1}{2\lambda}\Eqsmall{x}{q}\brs{\dx \log \p(\mb{x})k(\mb{x}, \cdot) + \dx k(\mb{x}, \cdot)} - \frac{1}{2\lambda}\sum^m_{i'=1} \nu^{(l)}_{i'} k(\mb{x}'_{i'}, \cdot) =-\frac{z_l}{2\lambda},
\]
where for convenience we have denoted $z_l = \Eqsmall{x}{q}\brs{\dx \log \p(\mb{x})k(\mb{x}, \cdot) + \dx k(\mb{x}, \cdot)} + \sum^m_{i'=1} \nu^{(l)}_{i'} k(\mb{x}'_{i'}, \cdot)$. Substituting this back into \eqref{eq:dual2} gives
\begin{equation}
\mathcal{L}(\mb\theta, \mb{g}^{\star}, \mb\nu^{(1)}, \dots, \mb\nu^{(d)}, \lambda) = \sum^d_{l=1} \brangle{-\frac{z_l}{2\lambda}, z_l}_{\mathcal{G}} + \lambda \br{\sum^d_{i=1}\brangle{-\frac{z_l}{2\lambda}, -\frac{z_l}{2\lambda}}_{\mathcal{G}} - 1} = \sum^d_{l=1}\frac{1}{4\lambda}\brangle{z_l, z_l}_{\mathcal{G}}+\lambda,
\label{eq:dual3}
\end{equation}
where $\mb{g}^{\star} = (g_1^{\star}, \dots, g_d^{\star})$. Solve for $\lambda$ by differentiating
\[
\frac{\partial \mathcal{L}}{\partial \lambda} = -\frac{1}{4\lambda^2}\sum^d_{i=1}\brangle{z_l, z_l}_{\mathcal{G}}+1 =0.
\]
Rearranging for $\lambda$ gives $\lambda = \pm \sqrt{{\sum^d_{l=1}\brangle{z_l, z_l}_{\mathcal{G}}}/{4}}$. Since $\lambda \geq 0$, we take the positive solution, giving $\lambda^{\star} = \sqrt{\sum^d_{l=1}\brangle{z_l, z_l}_{\mathcal{G}}}/2$. Substituting $\lambda^{\star}$ into \eqref{eq:dual3} gives
\begin{equation}
\mathcal{L}(\mb\theta, \mb{g}^{\star}, \mb\nu^{(1)}, \dots, \mb\nu^{(d)}, \lambda^{\star}) = \sum^d_{l=1}\frac{1}{4\lambda^{\star}}\brangle{z_l, z_l}_{\mathcal{G}}+\lambda^{\star} = \sqrt{\sum^d_{l=1}\brangle{z_l, z_l}_{\mathcal{G}}}.
\label{eq:dual4}
\end{equation}
To solve for the final Lagrangian parameters, $\mb\nu^{(1)}, \dots, \mb\nu^{(d)}$,  let us first introduce some notation. Denote
\[
\mb\phi_{\bm{x}'} = \begin{bmatrix}
k(\mb{x}_1', \cdot), \\
\vdots, \\
k(\mb{x}_m', \cdot) 
\end{bmatrix}, \qquad  \kxxp{z} = \begin{bmatrix}
k(\mb{z}, \mb{x}_1') &
\hdots &
k(\mb{z}, \mb{x}_m') 
\end{bmatrix}, \qquad \bm{K}' = \mb\phi_{\bm{x}'}\mb\phi_{\bm{x}'}\T,
\]
where $\mb{x}'_1, \dots, \mb{x}'_m \in \approxdV$. Now consider the equivalence
\[
{\mb{\nu}^{(l)}}^{\star} := \argmin_{\mb\nu^{(l)}}\sqrt{\sum^d_{l=1}\brangle{z_l, z_l}_{\mathcal{G}}}= \argmin_{\mb\nu^{(l)}}\sum^d_{l=1}\brangle{z_l, z_l}_{\mathcal{G}}= \argmin_{\mb\nu^{(l)}}\brangle{z_l, z_l}_{\mathcal{G}},
\]
as each $\mb{\nu}^{(l)}$ is independent. By rewriting $z_l$ as
\(
z_l = \Eqsmall{x}{q}\brs{\dx \log \p(\mb{x})k(\mb{x}, \cdot) + \dx k(\mb{x}, \cdot)} + {\mb\nu^{(l)}}\T\mb\phi_{\bm{x}'} ,
\)
we solve this again by differentiating,
\begin{align*}
    \frac{d}{d\mb{\nu}^{(l)}}\brangle{z_l, z_l}_{\mathcal{G}} = 2 \brangle{\frac{dz_l}{d\mb\nu^{(l)}}, z_l}_{\mathcal{G}} &=2\brangle{\mb\phi_{\bm{x}'}, \Eqsmall{x}{q}\brs{\dx \log \p(\mb{x})k(\mb{x}, \cdot) + \dx k(\mb{x}, \cdot)} + {\mb\nu^{(l)}}\T\mb\phi_{\bm{x}'} }_{\mathcal{G}} \\
    &=2\;\Eqsmall{x}{q}\brs{\dx \log \p(\mb{x}) \kxxp{x} + \dx  \kxxp{x}} + 2 {\mb\nu^{(l)}}\T\mb\phi_{\bm{x}'}\mb\phi_{\bm{x}'}\T\\
    &=2\;\Eqsmall{x}{q}\brs{\dx \log \p(\mb{x}) \kxxp{x} + \dx  \kxxp{x}} + 2{\mb\nu^{(l)}}\T\bm{K}' .
\end{align*}
Setting equal to zero and rearranging gives
\begin{equation}
{\mb\nu^{(l)}}^{\star} = -(\bm{K}')^{-1}\Eqsmall{x}{q}\brs{\dx \log \p(\mb{x}) \kxxp{x}\T + (\dx \kxxp{x})\T}.
\label{eq:nustar}
\end{equation}
Before substituting back into \eqref{eq:dual4}, let us first square it and expand it as 
\begin{equation}
\begin{split}
    \mathcal{L}^2 &= \sum^d_{l=1} \Eq{x}{q}\Eq{y}{q} [u_l(\mb{x}, \mb{y})] + \br{\Eq{x}{q}\Big[ \dx \log \p(\mb{x}) \kxxp{x}\T + (\dx  \kxxp{x})\T \Big]}\T\mb\nu^{(l)}\\
    &\qquad\qquad\qquad\qquad\qquad  + {\mb\nu^{(l)}}\T\Eq{y}{q}\Big[ \dy \log \p(\mb{y}) \kxxp{y}\T + (\dy \kxxp{y})\T \Big]+ {\mb\nu^{(l)}}\T\bm{K}'{\mb\nu^{(l)}} \label{eq:expandeddual1}
\end{split}
\end{equation}
where
\(
u_l(\mb{x}, \mb{y}) =\scorel{p}(\mb{x})\scorel{p}(\mb{y})k(\mb{x},\mb{y}) +\scorel{p}(\mb{x}) \dy k(\mb{x}, \mb{y}) +\scorel{p}(\mb{y}) \dx k(\mb{x}, \mb{y}) + \dx\dy k(\mb{x}, \mb{y}).
\) 
We can substitute ${\mb\nu^{(l)}}^{\star}$ from \eqref{eq:nustar} into $\mathcal{L}^2$, denoting ${\mathcal{L}^{\star}}^2 := \mathcal{L}(\mb\theta, \mb{g}^{\star}, {\mb\nu^{(1)}}^{\star}, \dots, {\mb\nu^{(d)}}^{\star}, \lambda)$, giving
\begin{align}
    \mathcal{L}^2 &= \sum^d_{l=1} \Eq{x}{q}\Eq{y}{q} [u_l(\mb{x}, \mb{y})] + \br{\Eq{x}{q}\Big[ \scorel{\mb{x}} \kxxp{x}\T + (\dx  \kxxp{x})\T \Big]}\T\br{-(\bm{K}')^{-1}\Eq{y}{q}\Big[\scorel{\mb{y}} \kxxp{y}\T + (\dy  \kxxp{y})\T\Big]} \nonumber\\
    &\qquad \qquad + \br{-(\bm{K}')^{-1}\Eq{x}{q}\brs{\scorel{\mb{x}} \kxxp{x}\T + (\dx  \kxxp{x})\T}}\T\Eq{y}{q}\Big[ \scorel{\mb{y}} \kxxp{y}\T + (\dy  \kxxp{y})\T \Big] \nonumber \\
    &\qquad\qquad  + \br{-(\bm{K}')^{-1}\Eq{x}{q}\brs{\scorel{\mb{x}} \kxxp{x}\T + (\dx  \kxxp{x})\T}}\T\bm{K}'\br{-(\bm{K}')^{-1}\Eq{y}{q}\brs{\scorel{\mb{y}} \kxxp{y}\T + (\dy  \kxxp{y})\T}}, \nonumber 
\end{align}
where, for brevity, we have written $\scorel{\mb{x}} = \dx \log \p(\mb{x})$ and $\scorel{\mb{y}} = \dy \log \p(\mb{y})$. This can be further simplified to
\begin{equation}
    \mathcal{L}^2=\sum^d_{l=1} \Eq{x}{q}\Eq{y}{q} [u_l(\mb{x}, \mb{y})] - \br{\Eq{x}{q}\Big[ \scorel{\mb{x}} \kxxp{x}\T + (\dx  \kxxp{x})\T \Big]}\T(\bm{K}')^{-1}\Eq{y}{q}\Big[\scorel{\mb{y}} \kxxp{y}\T + (\dy  \kxxp{y})\T\Big],
\end{equation}
and equivalently
\begin{equation}
    \mathcal{L}^2=\sum^d_{l=1} \Eq{x}{q}\Eq{y}{q} \brs{u_l(\mb{x}, \mb{y}) - \br{\scorel{\mb{x}} \kxxp{x}\T + (\dx  \kxxp{x})\T}\T (\bm{K}')^{-1}\br{\scorel{\mb{y}} \kxxp{y}\T + (\dy  \kxxp{y})\T}},
\end{equation}

giving the desired result.

\subsection{Proof of \Cref{thm:theta_star}} \label{app:theta_star}

First, we state a general result for proving the consistency of an empirical estimator of $\mb \theta$, which is second-order differentiable with respect to $\mb \theta$. 

\begin{lemma}
\label{lem:tksd:consistency}
    Define the unique solution of empirical objective $\hat{\mb \theta} := \argmin_{\mb \theta} \hat{L}(\mb \theta)$ and the population $\mb \theta^{\star} := \argmin_{\mb \theta} L(\mb \theta)$, where $L(\mb\theta) := \mathbb{E}[\hat{L}(\mb\theta)]$. If $\lambda_\mathrm{min}\left[ \nabla^2_{\mb \theta} \hat{L}(\mb \theta) \right] \ge \Lambda_\mathrm{{min}} > 0, \forall \mb \theta$ and $\|\nabla_{\mb\theta} \hat{L}(\mb\theta^{\star})\| = O_P(\frac{1}{\sqrt{n}})$, then $\|\hat{\mb\theta} - {\mb \theta}^\star\| = O_P(\frac{1}{\sqrt{n}})$. Here, $\lambda_\mathrm{min} (M)$ is the smallest eigenvalue of a matrix $M$.
\end{lemma}
\begin{proof}
First, we define a function $h(\mb{\theta}) := \langle \hat{\mb{\theta}} - \mb{\theta}^{\star}, \hat{L}(\mb{\theta}) \rangle$. Using the mean value theorem, we obtain 
$h(\hat{ \mb{\theta}}) - h( \mb{\theta}^{\star}) = \langle  \nabla_{\mb{\theta}} h(\bar{\mb{\theta}}), \hat{\mb{\theta}} - \mb{\theta}^{\star} \rangle$, where $\bar{\mb\theta}$ is the midpoint between $\hat{ \mb{\theta}}$ and $\mb{\theta}^\star$ in a co-ordinate wise fashion. Therefore, due to the fact that $\nabla_{\mb{\theta}} \hat{L}(\hat{\mb{\theta}}) = 0 $, 
\begin{align*}
    0 - h(\mb{\theta}^{\star}) = \langle \hat{\mb{\theta}} - \mb{\theta}^{\star},  \nabla_{\mb{\theta}} \hat{L}(\mb{\theta}^{\star}) \rangle = \langle \hat{\mb{\theta}} - \mb{\theta}^{\star}, \nabla^2_{\mb{\theta}} \hat{L}(\bar{\mb{\theta}}) (\hat{\mb{\theta}} - \mb{\theta}^{\star}) \rangle.
\end{align*}
This implies the following inequality
\begin{align*}
    \|\hat{\mb{\theta}} - \mb{\theta}^{\star}\| \|\nabla_{\mb{\theta}} \hat{L}(\mb{\theta}^{\star})\| \ge \|(\hat{\mb{\theta}} - \mb{\theta}^{\star})^\top \mb{\nabla}_{\mb\theta}^2 \hat{L}(\hat{\mb{\theta}}) (\hat{\mb{\theta}} - \mb{\theta}^{\star}) \| \ge \lambda_\mathrm{min} \|(\hat{\mb{\theta}} - \mb{\theta}^{\star})\|^2. 
\end{align*}
Assume $\|\hat{\mb{\theta}} - \mb{\theta}^{\star}\| \neq 0$, 
$\frac{1}{\lambda_\mathrm{min}} \|{\nabla}_{\mb\theta} \hat{L}(\mb{\theta}^{\star})\| \ge \|\hat{\mb{\theta}} - \mb{\theta}^{\star}\|^2$.
\end{proof}

For the remainder, we need only to show that $\mathcal{D}_{\text{TKSD}}(\p | q)^2$ satisfies \cref{lem:tksd:consistency}. 
\begin{proof}

First, we need to show that

\begin{equation}
\left [\nabla_{\mb\theta} \mathcal{D}_{\text{TKSD}}(\p | q)^2 \Big|_{\mb\theta = \mb\theta^\star} \right]_i = O_P(\hat{\varepsilon}_m), \forall i. 
\label{eq:tksd:tksd_deriv}
\end{equation}
    
For brevity let us define $\E_{\mb{x}} = \Eq{x}{q}$ and similarly $\E_{\mb{y}} = \Eq{y}{q}$. Recall
\begin{equation}
\mathcal{D}_{\text{TKSD}}(\p | q)^2 = \sum^d_{l=1} \E_{\mb{x}  }\E_{\mb{y}  } \brs{u_l(\mb{x}, \mb{y}) - \bm{v}_l(\mb{x})\T (\bm{K}')^{-1}\bm{v}_l(\mb{y})}
\label{eq:tksd:tksd_sq_proof}
\end{equation}
where 
\begin{align*}
u_l(\mb{x}, \mb{y}) &=\scorel{p}(\mb{x})\scorel{p}(\mb{y})k(\mb{x},\mb{y}) +\scorel{p}(\mb{x}) \dy k(\mb{x}, \mb{y}) +\scorel{p}(\mb{y})  \dx k(\mb{x}, \mb{y}) + \dx\dy \mb{y}), \\
\bm{v}_l(\mb{z}) &=\scorel{p}(\mb{z})\kxxp{z}\T + (\dz \kxxp{z})\T,
\end{align*}
$ \kxxp{z} = [ k(\mb{z}, \mb{x}_1'), \dots, k(\mb{z}, \mb{x}_m') ]$, $\mb\phi_{\bm{x}'} = [k(\mb{x}_1', \cdot), \dots, k(\mb{x}_m', \cdot)]\T$ and $\bm{K}' = \mb\phi_{\bm{x}'}\mb\phi_{\bm{x}'}\T$.
Let us take the derivative of \eqref{eq:tksd:tksd_sq_proof} with respect to $\mb\theta$. We first consider for the $l$-th dimension of the sum, and then note that this will apply to all $d$ dimensions. Therefore consider
\begin{align}
\nabla_{\mb\theta}&\mathcal{D}_{\text{TKSD}}(\p | q)^2\Big|_{\mb\theta = \mb\theta^\star}  = \nabla_{\mb\theta}\E_{\mb{x}  }\E_{\mb{y}  } \brs{u_l(\mb{x}, \mb{y}) - \bm{v}_l(\mb{x})\T (\bm{K}')^{-1}\bm{v}_l(\mb{y})}\Big|_{\mb\theta = \mb\theta^\star} \nonumber \\
&= \E_{\mb{x}  }\E_{\mb{y}  } \brs{\nabla_{\mb\theta}u_l(\mb{x}, \mb{y})  - \mb{v}_l(\mb{x})\T (\bm{K}')^{-1}(\nabla_{\mb\theta}\bm{v}_l(\mb{y})) - (\nabla_{\mb\theta}\mb{v}_l(\mb{x}))\T (\bm{K}')^{-1}\bm{v}_l(\mb{y})}\Big|_{\mb\theta = \mb\theta^\star} \label{eq:tksd:bigtksdderiv}.
\end{align}
%Note that $u_l(\mb{x}, \mb{y})$ can be considered as the `KSD part' of TKSD, and $\bm{v}_l(\mb{x})\T (\bm{K}')^{-1}\bm{v}_l(\mb{y})$ can be considered as an additional part to account for truncation.
We begin by expanding the first term in \eqref{eq:tksd:bigtksdderiv}:
\begin{align}
\E_{\mb{x}  }\E_{\mb{y}  } \brs{\nabla_{\mb\theta}u_l(\mb{x}, \mb{y})}\Big|_{\mb\theta = \mb\theta^\star} &= \E_{\mb{x}  }\E_{\mb{y}  } \big[\scorel{p}(\mb{x})k(\mb{x},\mb{y})(\nabla_{\mb\theta}\scorel{p}(\mb{y})) 
 + (\nabla_{\mb\theta}\scorel{p}(\mb{x}))k(\mb{x},\mb{y})\scorel{p}(\mb{y})\nonumber \\ 
&\quad\qquad + (\nabla_{\mb\theta}\scorel{p}(\mb{x}))  \dy k(\mb{x}, \mb{y}) + (\nabla_{\mb\theta}\scorel{p}(\mb{y}))  \dx k(\mb{x}, \mb{y})\big]\Big|_{\mb\theta = \mb\theta^\star}. \label{eq:tksd:ksdderivpart0}
\end{align}
Now expand the first term of \eqref{eq:tksd:ksdderivpart0}, giving
\begin{align}
\E_{\mb{x}  }\E_{\mb{y}  } \brs{\scorel{p}(\mb{x})k(\mb{x},\mb{y})(\nabla_{\mb\theta}\scorel{p}(\mb{y}))}\Big|_{\mb\theta = \mb\theta^\star}
&= \int_{V} q(\mb{x})\scorel{p}(\mb{x}) \E_{\mb{y}}\brs{k(\mb{x},\mb{y})(\nabla_{\mb\theta}\scorel{p}(\mb{y}))} d\mb{x} \Big|_{\mb\theta = \mb\theta^\star}   \nonumber\\
&=  \int_{V} \dx q(\mb{x})\E_{\mb{y}} \brs{k(\mb{x},\mb{y}) [\nabla_{\mb\theta}\scorel{p}(\mb{y})]\big|_{\mb\theta = \mb\theta^\star}} d\mb{x} \label{eq:tksd:ksdderivpart1}
\end{align}
due to
\[
q(\mb{z})\scorel{p}(\mb{z})\Big|_{\mb\theta = \mb\theta^\star} = q(\mb{z})  \dy \log p_{\mb{\theta}^\star}(\mb{z}) = q(\mb{z})\frac{ \dy q(\mb{z})}{q(\mb{z})} =  \dy q(\mb{z}).
\]
Then, via integration by parts, \eqref{eq:tksd:ksdderivpart1} becomes
\begin{align*}
    \oint_{\partial V} q(\mb{x}) \E_{\mb{y}}\brs{k(\mb{x},\mb{y}) [\nabla_{\mb\theta}\scorel{p}(\mb{y})]\big|_{\mb\theta = \mb\theta^\star}}\hat{\mathrm{u}}_l(\mb{x}) ds - \int_V q(\mb{x})\brs{\dx k(\mb{x},\mb{y}) [\nabla_{\mb\theta}\scorel{p}(\mb{y})]\big|_{\mb\theta = \mb\theta^\star}}.
\end{align*}

We can expand the second term of \eqref{eq:tksd:ksdderivpart0} in a similar way: 
\[
\begin{split}
(\nabla_{\mb\theta}\scorel{p}(\mb{x}))k(\mb{x},\mb{y})\scorel{p}(\mb{y}) &= \oint_{\partial V} q(\mb{y}) \E_{\mb{x}}\brs{[\nabla_{\mb\theta}\scorel{p}(\mb{x})]\big|_{\mb\theta = \mb\theta^\star}k(\mb{x},\mb{y}) }\hat{\mathrm{u}}_l(\mb{y}) ds \\
&\qquad\qquad\qquad- \int_V q(\mb{y})\brs{[\nabla_{\mb\theta}\scorel{p}(\mb{x})]\big|_{\mb\theta = \mb\theta^\star} \dy k(\mb{x},\mb{y})}
\end{split}
\]
Both of these together, as they appear in \eqref{eq:tksd:ksdderivpart0}, gives 
\begin{equation}
\begin{split}
B := &\oint_{\partial V} q(\mb{x}) \E_{\mb{y}}\brs{k(\mb{x},\mb{y}) [\nabla_{\mb\theta}\scorel{p}(\mb{y})]\big|_{\mb\theta = \mb\theta^\star}}\hat{\mathrm{u}}_l(\mb{x}) ds \\
&\;\;\;\;\;\;\;\;\;+  \oint_{\partial V} q(\mb{y}) \E_{\mb{x}}\brs{[\nabla_{\mb\theta}\scorel{p}(\mb{x})]\big|_{\mb\theta = \mb\theta^\star}k(\mb{x},\mb{y}) }\hat{\mathrm{u}}_l(\mb{y}) ds.
\end{split}
\label{eq:tksd:B}
\end{equation}

Now consider the second term in \eqref{eq:tksd:bigtksdderiv}, 
\begin{align}
\E_{\mb{x}}\E_{\mb{y}}&\brs{\mb{v}_l(\mb{x})\T (\bm{K}')^{-1}(\nabla_{\mb\theta}\bm{v}_l(\mb{y}))}\Big|_{\mb\theta = \mb\theta^\star} \\
&= \E_{\mb{x}}\E_{\mb{y}}\brs{(\scorel{p}(\mb{x})\kxxp{x} + \dx \kxxp{x}) (\bm{K}')^{-1}(\kxxp{y}\T(\nabla_{\mb\theta}\scorel{p}(\mb{y}))\T)}\Big|_{\mb\theta = \mb\theta^\star} \nonumber \\
&= \E_{\mb{x}}\E_{\mb{y}}\Big[\scorel{p}(\mb{x})\kxxp{x} (\bm{K}')^{-1}(\kxxp{y}\T(\nabla_{\mb\theta}\scorel{p}(\mb{y}))\T) \\
&\qquad\qquad\qquad+ (\dx \kxxp{x})(\bm{K}')^{-1}(\kxxp{y}\T(\nabla_{\mb\theta}\scorel{p}(\mb{y}))\T))\Big]\Big|_{\mb\theta = \mb\theta^\star}\nonumber \\
&= \E_{\mb{x}}\brs{\scorel{p}(\mb{x})\kxxp{x} (\bm{K}')^{-1}\E_{\mb{y}}\brs{\kxxp{y}\T(\nabla_{\mb\theta}\scorel{p}(\mb{y})\T)}}\Big|_{\mb\theta = \mb\theta^\star}\nonumber \\
&\qquad\qquad\qquad+ \E_{\mb{x}}\brs{(\dx \kxxp{x})(\bm{K}')^{-1}\E_{\mb{y}}\brs{\kxxp{y}\T(\nabla_{\mb\theta}\scorel{p}(\mb{y}))\T)}}\Big|_{\mb\theta = \mb\theta^\star} \label{eq:tksd:tksd_deriv_part0}
\end{align}

Consider only the first term of the above,
\begin{align*}
&\E_{\mb{x}}\brs{\scorel{p}(\mb{x})\kxxp{x} (\bm{K}')^{-1}\E_{\mb{y}}\brs{(\kxxp{y}\T(\nabla_{\mb\theta}\scorel{p}(\mb{y}))\T)}}\Big|_{\mb\theta = \mb\theta^\star} \\
&\qquad\qquad\qquad= \int_{V} q(\mb{x})\scorel{p}(\mb{x})\kxxp{x} (\bm{K}')^{-1}\E_{\mb{y}}\brs{\kxxp{y}\T(\nabla_{\mb\theta}\scorel{p}(\mb{y}))\T}d\mb{x}\Big|_{\mb\theta = \mb\theta^\star} \\
&\qquad\qquad\qquad= \int_{V} \dx q(\mb{x}) \kxxp{x} (\bm{K}')^{-1}\E_{\mb{y}}\brs{\kxxp{y}\T[\nabla_{\mb\theta}\scorel{p}(\mb{y})\T]\big|_{\mb\theta = \mb\theta^\star}}d\mb{x}.
\end{align*}
Integration by parts can be used to expand this to
\begin{equation}
\begin{split}
\oint_{\partial V} q(\mb{x})&\kxxp{x} (\bm{K}')^{-1}\E_{\mb{y}}\brs{\kxxp{y}\T[\nabla_{\mb\theta}\scorel{p}(\mb{y})\T]\big|_{\mb\theta = \mb\theta^\star}}\hat{\mathrm{u}}_l(\mb{x}) ds \\
&\qquad\qquad\qquad -  \int_V q(\mb{x})\kxxp{x} (\bm{K}')^{-1}\E_{\mb{y}}\brs{\dx\kxxp{y}\T[\nabla_{\mb\theta}\scorel{p}(\mb{y})\T]\big|_{\mb\theta = \mb\theta^\star}} \nu(\mb{x}) d\mb{x}.
\end{split}
\label{eq:tksd:tksd_deriv_part1}
\end{equation}

\Cref{ass.least.square} indicates that we have the following equivalence (in a coordinate-wise fashion): 
\begin{equation}
\begin{split}
\oint_{\partial V} q(\mb{x})\kxxp{x} &(\bm{K}')^{-1}\E_{\mb{y}}\brs{\kxxp{y}\T[\nabla_{\mb\theta}\scorel{p}(\mb{y})\T]\big|_{\mb\theta = \mb\theta^\star}}\hat{\mathrm{u}}_l(\mb{x}) ds \\
&= \oint_{\partial V} q(\mb{x})\E_{\mb{y}}\brs{k(\mb{x}, \mb{y})[\nabla_{\mb\theta}\scorel{p}(\mb{y})\T]\big|_{\mb\theta = \mb\theta^\star}}\hat{\mathrm{u}}_l(\mb{x}) ds + O_P(\hat\varepsilon_{m,1}).
\end{split}
\label{eq:tksd:approx1}
\end{equation}
We interpret this as follows. If we consider $\bm{x}' \in \approxdV$ as our training set, then our regression function is trained on the approximate set of boundary points. The surface integral denoted by $\oint_{\partial V}$ is evaluating on all $\mb{x} \in \partial V$, the true boundary. Therefore the prediction based on $\mb{x} \in \partial V$ is going to be well approximated by the regression function that is trained on $\bm{x}'$, and the approximation error, $\hat\varepsilon_{m,1}$, will decrease as $m$, the number of training points, increases. 

We can apply the same operations (from \eqref{eq:tksd:tksd_deriv_part0} onwards) to the third term in \eqref{eq:tksd:bigtksdderiv}, which gives
\begin{equation}
\E_{\mb{x}}\E_{\mb{y}}\brs{(\nabla_{\mb\theta}\mb{v}_l(\mb{x}))\T (\bm{K}')^{-1}\bm{v}_l(\mb{y})} =  \oint_{\partial V} q(\mb{y})\E_{\mb{x}}\brs{[\nabla_{\mb\theta}\scorel{p}(\mb{x})\T]\big|_{\mb\theta = \mb\theta^\star}k(\mb{x}, \mb{y})}\hat{\mathrm{u}}_l(\mb{y}) ds + O_P(\hat\varepsilon_{m,2}),
\label{eq:tksd:approx2}
\end{equation}
where $\hat\varepsilon_{m,2}$ is the approximation error for the corresponding kernel regression on this term. When taking the sum, as it is in \eqref{eq:tksd:bigtksdderiv}, \eqref{eq:tksd:approx1} + \eqref{eq:tksd:approx2},
\[
\begin{split}
\oint_{\partial V} q(\mb{x})&\E_{\mb{y}}\brs{k(\mb{x}, \mb{y})[\nabla_{\mb\theta}\scorel{p}(\mb{y})\T]\big|_{\mb\theta = \mb\theta^\star}}\hat{\mathrm{u}}_l(\mb{x}) ds + O_P(\hat\varepsilon_{m,1}) \\
& + \oint_{\partial V} q(\mb{y})\E_{\mb{x}}\brs{[\nabla_{\mb\theta}\scorel{p}(\mb{x})\T]\big|_{\mb\theta = \mb\theta^\star}k(\mb{x}, \mb{y})}\hat{\mathrm{u}}_l(\mb{y}) ds + O_P(\hat\varepsilon_{m,2}) = B + O_P(\hat\varepsilon_{m}),
\end{split}
\]
where $\hat\varepsilon_m = \hat\varepsilon_{m,1} + \hat\varepsilon_{m,2}$. Therefore, when substituting all of \eqref{eq:tksd:B}, \eqref{eq:tksd:approx1} and \eqref{eq:tksd:approx2} back into \eqref{eq:tksd:bigtksdderiv}, we have 
\[
\nabla_{\mb\theta}\mathcal{D}_{\text{TKSD}}(\p | q)^2\Big|_{\mb\theta = \mb\theta^\star} = O_P(\hat\varepsilon_{m}),
\]
where $\hat\varepsilon_{m}$ is a decreasing function of $m$.

Now we show $\mb{\theta}^{\star}$ is the true density parameter, i.e., $p_{\mb{\theta}^{\star}} = q$. Let $\mathcal{D}_{\text{TKSD}}(p_{\mb\theta^\star} | q) \defeq \mathcal{D}_{\text{TKSD}}(\p | q)|_{\mb\theta = \mb\theta^\star}$. First, we show that $L(\mb{\theta}^{\star}) = 0$. This is guaranteed as 
\begin{align*}
L(\mb{\theta}^{\star}) &= \lim_{m \to \infty}\mathcal{D}_{\text{TKSD}}(p_{\mb\theta^\star} | q)^2 \\
&= \lim_{m\to \infty} \left\{\sup_{\mb{g} \in \approxG} \mathbb{E}_q [\mathcal{T}_{p_{\mb{\theta}^{\star}}} \mb{g}]  \right\}^2 \\
&= \lim_{m \to \infty} O_P(\varepsilon_m) = 0 \text{ (with high probability)},
\end{align*}
for $\mb{g} \in \approxG$. The change from Stein discrepancy to $O_P(\varepsilon_m)$ is guaranteed by Theorem \ref{thm:steinidentityGtilde}. Since we assume that the minimiser of the population objective $\mb{\theta}^{\star}$ is unique and our density model is identifiable, it shows that the unique solution of the population objective is the unique optimal parameter of the density model. 

Second, we verify that $\|\nabla_{\mb\theta} \hat{L}(\mb{\theta}^{\star})\| = O_P(\frac{1}{\sqrt{n}})$.
\begin{align}
    \|\nabla_{\mb\theta} \hat{L}(\mb{\theta}^{\star})\| &= \| \nabla_{\mb\theta} \lim_{m \to \infty}  \widehat{\mathcal{D}}_{\text{TKSD}}(p_{\mb\theta^\star} | q)^2\|\notag \\
    &= \lim_{m \to \infty} \|\nabla_{\mb\theta} \widehat{\mathcal{D}}_{\text{TKSD}}(p_{\mb\theta^\star} | q)^2\|\notag\\
    &= \lim_{m \to \infty} \|\nabla_{\mb\theta} \widehat{\mathcal{D}}_{\text{TKSD}}(p_{\mb\theta^\star} | q)^2 - \nabla_{\mb\theta} \mathcal{D}_{\text{TKSD}}(p_{\mb\theta^\star} | q)^2 + \nabla_{\mb\theta} \mathcal{D}_{\text{TKSD}}(p_{\mb\theta^\star} | q)^2 \| \nonumber\\
    &\leq \lim_{m \to \infty} \br{\|\nabla_{\mb\theta} \widehat{\mathcal{D}}_{\text{TKSD}}(p_{\mb\theta^\star} | q)^2 - \nabla_{\mb\theta} \mathcal{D}_{\text{TKSD}}(p_{\mb\theta^\star} | q)^2\| + \|\nabla_{\mb\theta} \mathcal{D}_{\text{TKSD}}(p_{\mb\theta^\star} | q)^2 \|} \nonumber\\
    &\leq \lim_{m \to \infty} \br{\|\nabla_{\mb\theta} \widehat{\mathcal{D}}_{\text{TKSD}}(p_{\mb\theta^\star} | q)^2 - \nabla_{\mb\theta} \mathcal{D}_{\text{TKSD}}(p_{\mb\theta^\star} | q)^2\| + O_P(\hat{\varepsilon}_m) } \nonumber\\
    &\myleq{(a)} \lim_{m \to \infty} \left( O_P\br{\frac{1}{\sqrt{n}}} + O_P(\hat{\varepsilon}_m)\right)\\
    & \myeq{(b)} O_P\br{\frac{1}{\sqrt{n}}}, \notag
\end{align}
where the inequality in (a) is due to the convergence of U-statistics \citep{serfling2009}, and the equality in (b) is due to $\lim_{m\to \infty} O_P(\hat\varepsilon_m) = 0$.

\end{proof}

We have $\|\nabla_{\mb\theta} \hat{L}(\mb{\theta}^{\star})\| = O_P(\frac{1}{\sqrt{n}})$, and \eqref{ass.eigenvalue} provides $\lambda_\mathrm{min}\left[ \nabla^2_{\mb \theta} \hat{L}(\mb \theta) \right] \ge \Lambda_\mathrm{{min}} > 0$, then both conditions in \cref{lem:tksd:consistency} are satisfied, and this concludes the proof of the consistency of $\hat{\mb \theta}_n$, i.e., $\|\hat{\mb\theta}_{n} - \mb\theta^\star\| = O_P(\frac{1}{\sqrt{n}})$.

\section{Computation}

\subsection{Empirical Convergence of $\tilde{\mb{g}}$ to $\mb{g}$}

In \cref{lem:gtilde_to_g} we proved that under some conditions, the difference between $\mb{g} \in \mathcal{G}^d_0$ and $\tilde{\mb{g}} \in \approxG$ evaluated on $\mb{x}' \in \partial V$ is bounded in  probability by $\varepsilon_m$ and this probability tends to zero as $m \to \infty$. In this section we aim to show that $\tilde{\mb{g}}$ converges to a specific function as $m$ increases, which we hypothesise is `true' $\mb{g} \in \mathcal{G}^d_0$.

Figure \ref{fig:empirical_g} demonstrates empirical convergence of $\tilde{\mb{g}}$ as $m$ increases for a given dataset. The experiment is setup as follows: we simulate points from a $\mathcal{N}(\mb{0}_2, \mb{I}_2)$ distribution (a unit Multivariate Normal distribution in 2D), then truncate data points to within the shape of the boundary based on location, and repeat until we acquire $n=150$ truncated points. For each value of $m$, we minimise the TKSD objective to estimate $\mb\theta = \mb{0}_2$, and plug in the estimated $\hat{\mb\theta}$ to the formula for $\tilde{\mb{g}}$.

For lower values of $m$, there are not enough boundary points to enforce the constraint well on $\tilde{\mb{g}}$, and the approximate Stein identity has a high error, and the estimator is poor. The evaluations of $\tilde{\mb{g}}$ across the space do not match the $\tilde{\mb{g}}$ output for higher values of $m$. As $m$ increases, $\tilde{\mb{g}}$ begins to converge to a particular shape, and the difference between function evaluations for $m=150$ and $m=300$ is small.

\begin{figure}[t!]
    \centering
    \includegraphics[width=0.95\linewidth]{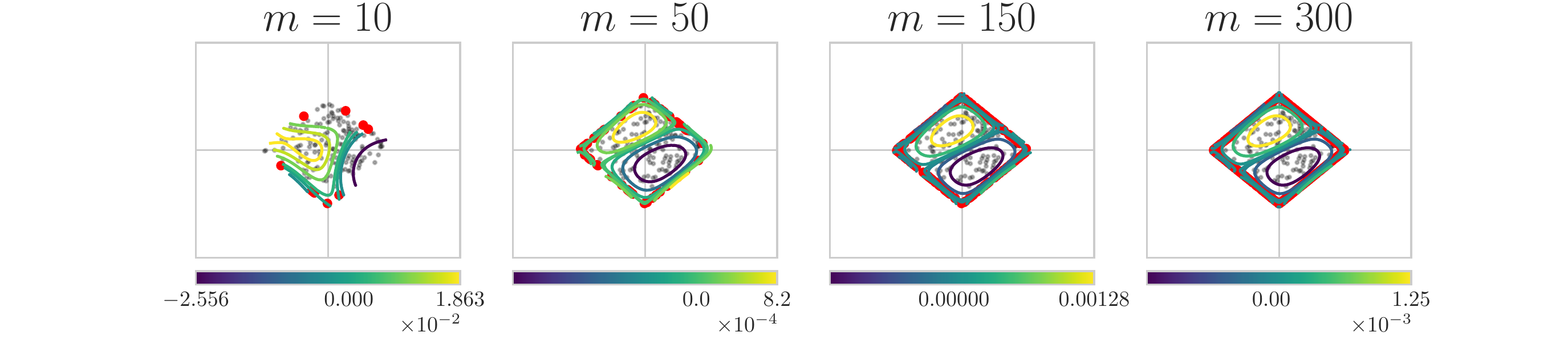}
    \includegraphics[width=0.95\linewidth]{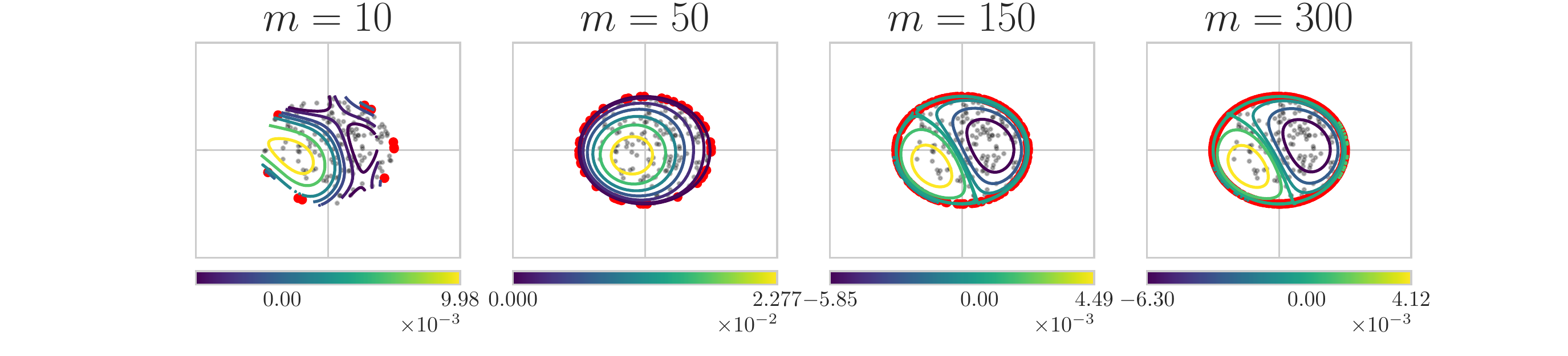}
    \includegraphics[width=0.95\linewidth]{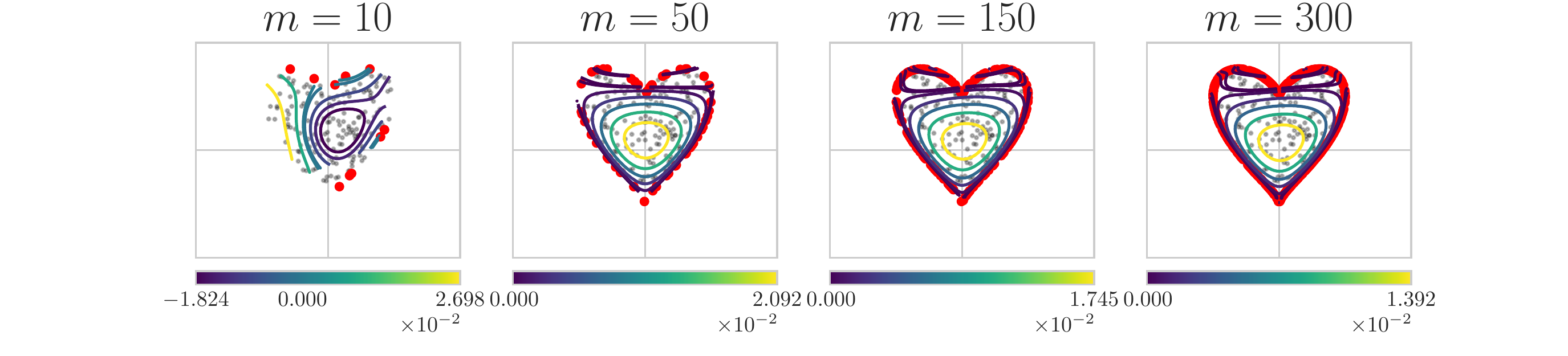}
    \caption{Contour lines of the optimal $\tilde{g}_0$ (first dimension of $\tilde{\mb{g}}$) output by TKSD across different values of $m$ and differently shaped truncation boundaries: the $\ell_1$ ball (top), the $\ell_2$ ball (middle) and a heart shape (bottom). Red points are all $m$ points in $\approxdV$ and grey points are samples from the truncated dataset. Note that we plot only the first dimension of $\tilde{\mb{g}}$ (i.e. $g_1$), but we observe the same pattern with the second dimension.}
    \label{fig:empirical_g}
\end{figure}

\subsection{Evaluating increasing values of $m$ in \cref{lem:dense}} \label{app:epsilon_decreasing}

\begin{figure}[ht!]
    \centering
    \includegraphics[width=\linewidth]{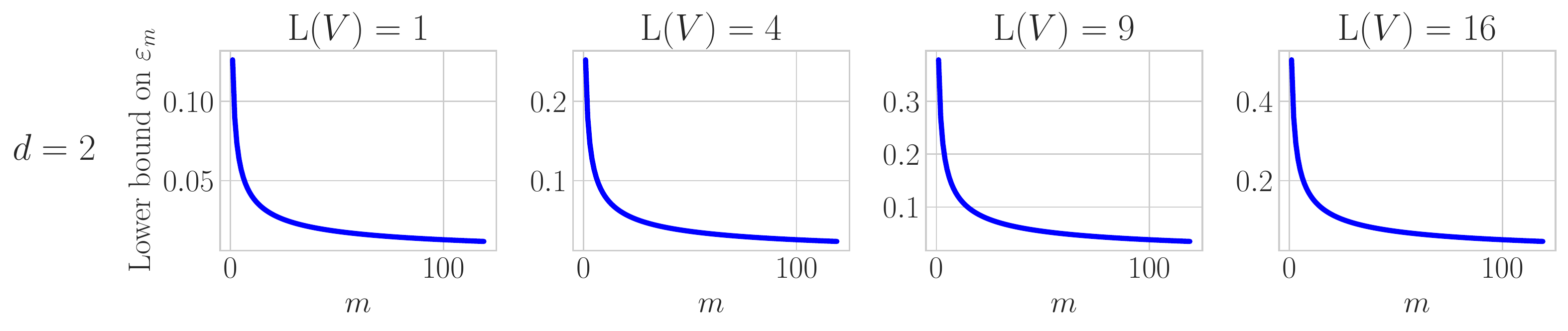}
    \includegraphics[width=\linewidth]{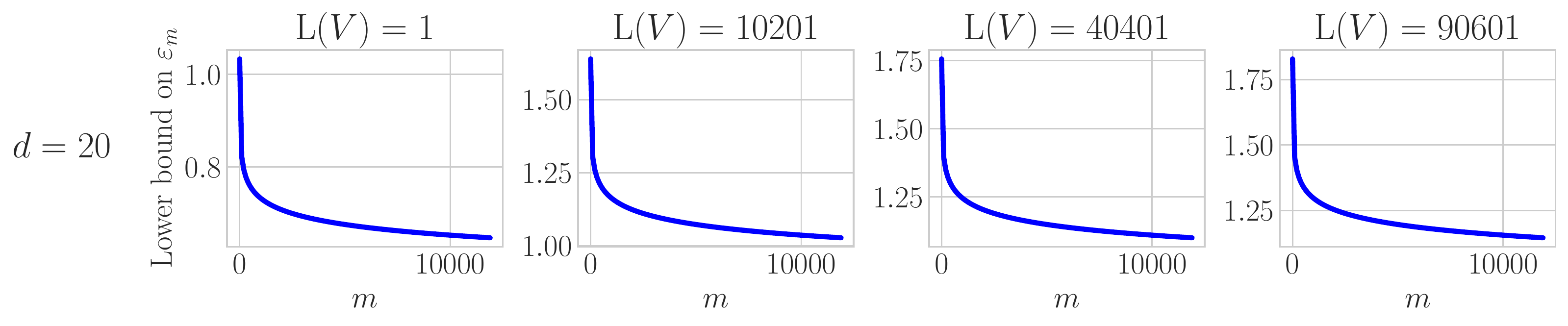}
    \includegraphics[width=\linewidth]{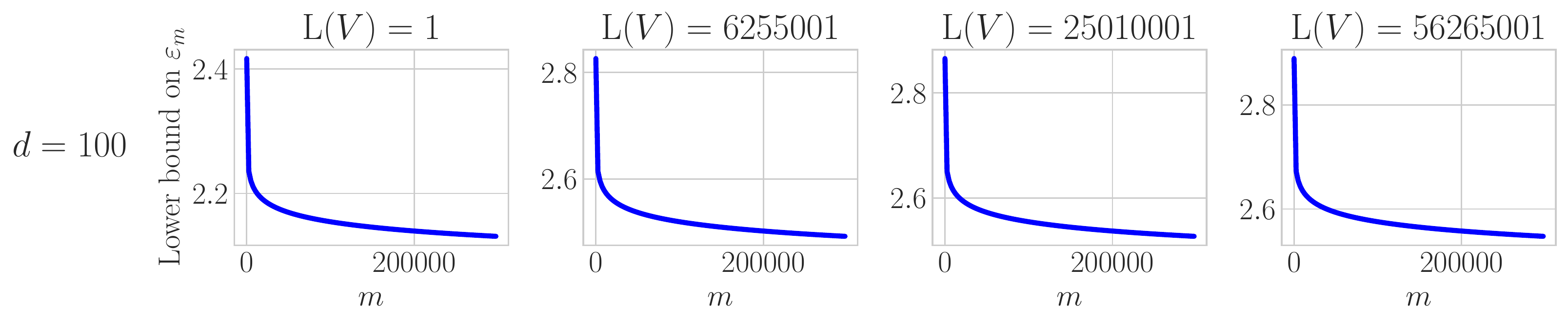}
    \caption{Lower bound on $\varepsilon_m$ (given in \eqref{eq:epsilon}) against $m$, the number of finite boundary points, plotted for different values of fixed dimension $d$ and boundary `size' $L(V)$, which scales quadratically.}
    \label{fig:decreasing_epsilon}
\end{figure}

In \cref{fig:decreasing_epsilon} we show that $\varepsilon_m$, as defined in \cref{lem:dense}, is bounded below by a decreasing function of $m$. In this example, we plot the value of the lower bound in \eqref{eq:epsilon} against $m$, the number of samples on the boundary set, and $\mathrm{L}(V)$, the $(d-1)$-surface area of the bounded domain $V$, which can be considered as a proxy to measure the complexity of the boundary. 

We let $\mathrm{L}(V)$ scale quadratically with dimension $d$, and treat it as a constant. $m$ scales linearly, and measure the effect on the lower bound of $\varepsilon_m$. As $\mathrm{L}(V)$ increases, the bound on $\varepsilon_m$ increases slightly. As $m$ increases, the bound on $\varepsilon_m$ decreases rapidly. Even at extremely high values of $\varepsilon_m$, relatively small values of $m$ are required to shrink the bound. This result highlights how to choose $m$ in experiments - as $m$ increases, the `denseness' of $\approxdV$ in $\partial V$ also increases, as expected. Whilst larger values of $m$ will always be better, there is a less pronounced effect when $m$ becomes significantly larger than zero. As discussed in \cref{app:extracomp}, larger $m$ comes with computational cost, and therefore this experiment shows that it may be beneficial to choose smaller values of $m$. This is explored further in \cref{app:consistency}.

\subsection{Empirical Consistency} \label{app:consistency}

\begin{figure}[!ht]
    \centering
    \includegraphics[width=0.51\linewidth]{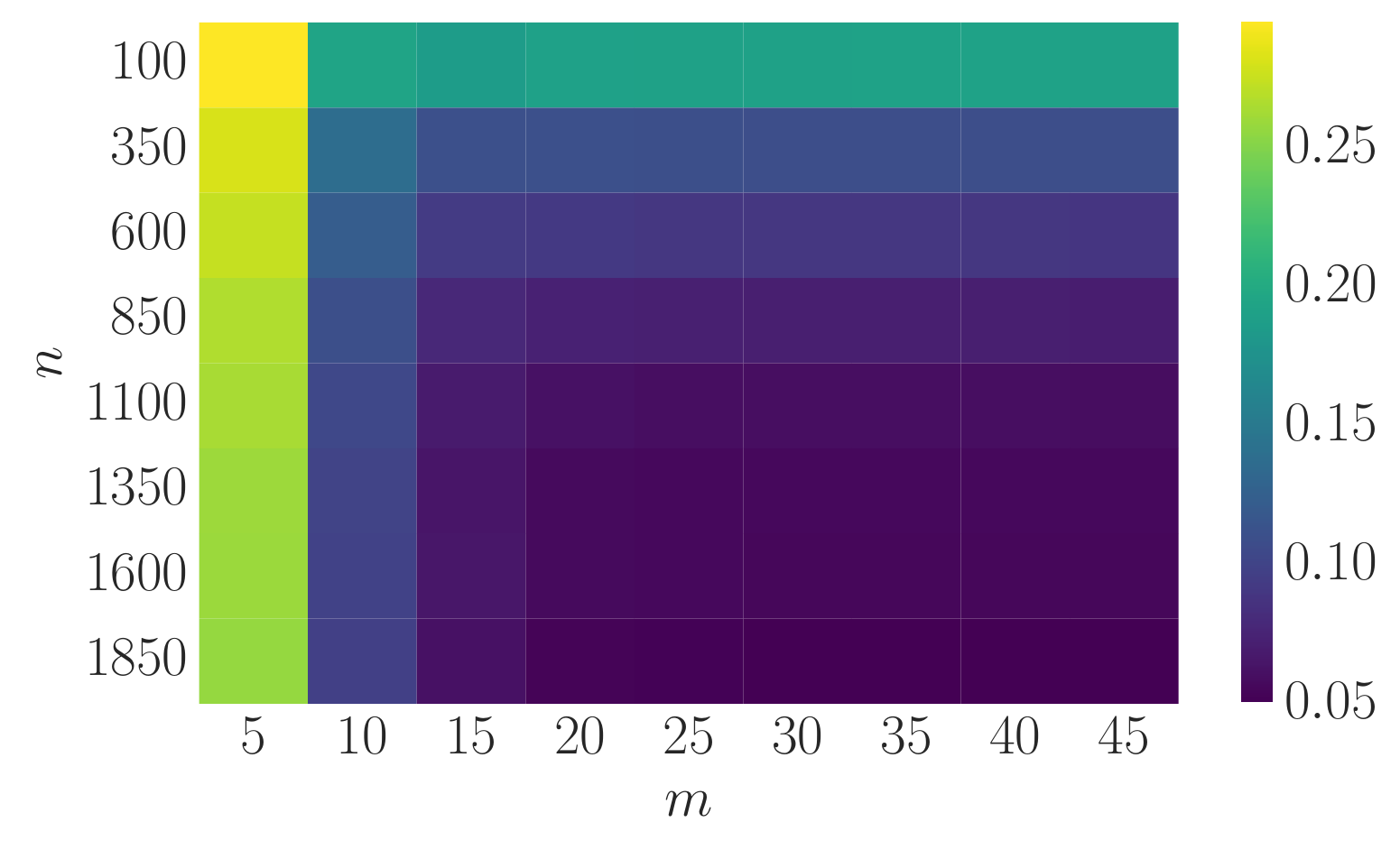}
    \includegraphics[width=0.45\linewidth]{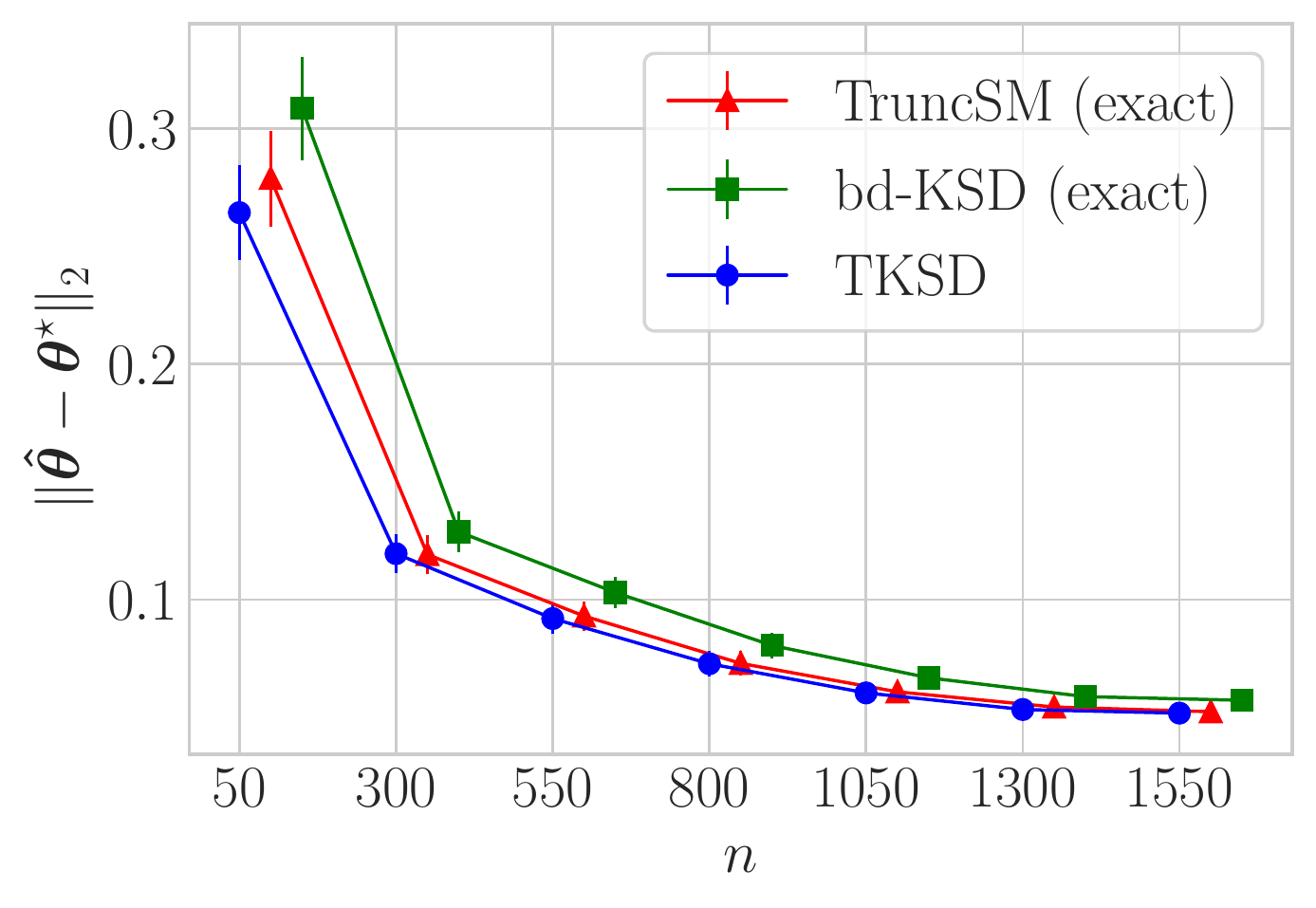}
    \caption{Left: Estimation error as $n$ and $m$ increases for TKSD only. Right: Mean estimation error for the three methods: TKSD, \textit{TruncSM} and bd-KSD, with standard error bars. TKSD uses a fixed $m=32$ across all values of $n$. Both plots report statistics over 64 seeds.}
    \label{fig:consistency}
\end{figure}

We verify that \eqref{eq.estimator} is consistent estimator via empirical experiments. Note that for consistency as proven in \cref{thm:theta_star}, we require taking the limit as $m \to \infty$, after which the estimator is consistent for $n$. So as $m$ and $n$ increases, the estimation error decreases towards zero. We show consistency as $m$ and $n$ increase empirically in \cref{fig:consistency}, for a simple experiment setup, similar to the setup from \cref{sec:dimensionbench}. Data are simulated from $\mathcal{N}(\mb{\mu}^\star, I_d)$, where $d=2$, $\mb{\mu}^\star = [0.5, 0.5]\T$, and $I_d$ is known. Data are truncated to the $\ell_2$ ball until we reach $n$ many data points (after truncation). 

The aim of estimation is $\mb{\mu}^\star$. Across 64 trials, \cref{fig:consistency} shows plots of the mean estimation error, given by $\|\mb{\hat{\mu}} - \mb{\mu}^\star\|$, where $\mb{\hat{\mu}}$ is the corresponding estimate of $\mb{\mu}^\star$ output by a given method. In the first plot (left), we show that as both $n$ and $m$ increase, the estimation error for TKSD decreases towards zero. In the second plot (right), for a fixed $m$, we show that the rate of convergence as $n$ increases for TKSD matches that of bd-KSD and \textit{TruncSM}.

\subsection{Gaussian Mixture Experiment}\label{app:mixture}

\begin{figure}[!ht]
\centering
\includegraphics[width=\linewidth]{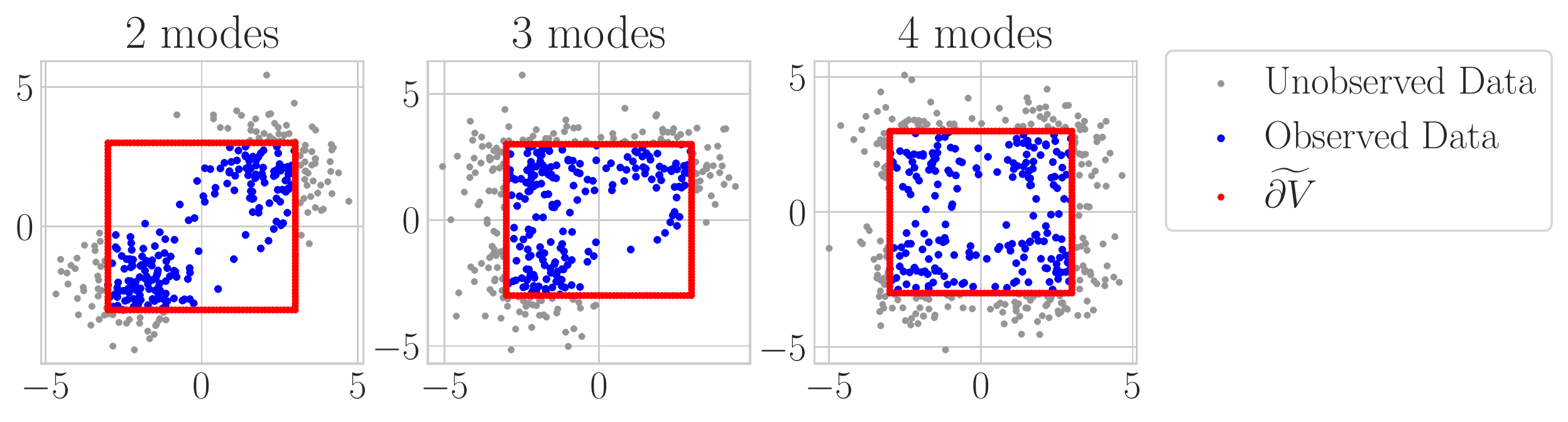}
\includegraphics[width=0.6\linewidth]{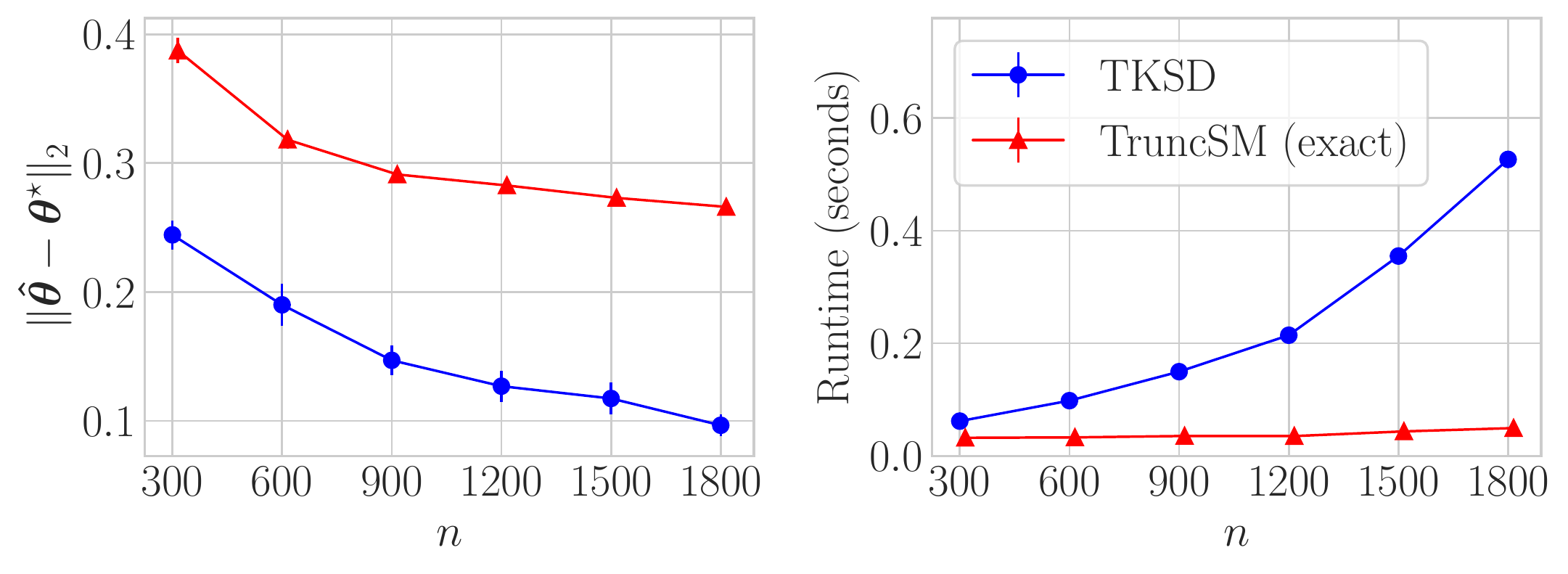}
\includegraphics[width=0.6\linewidth]{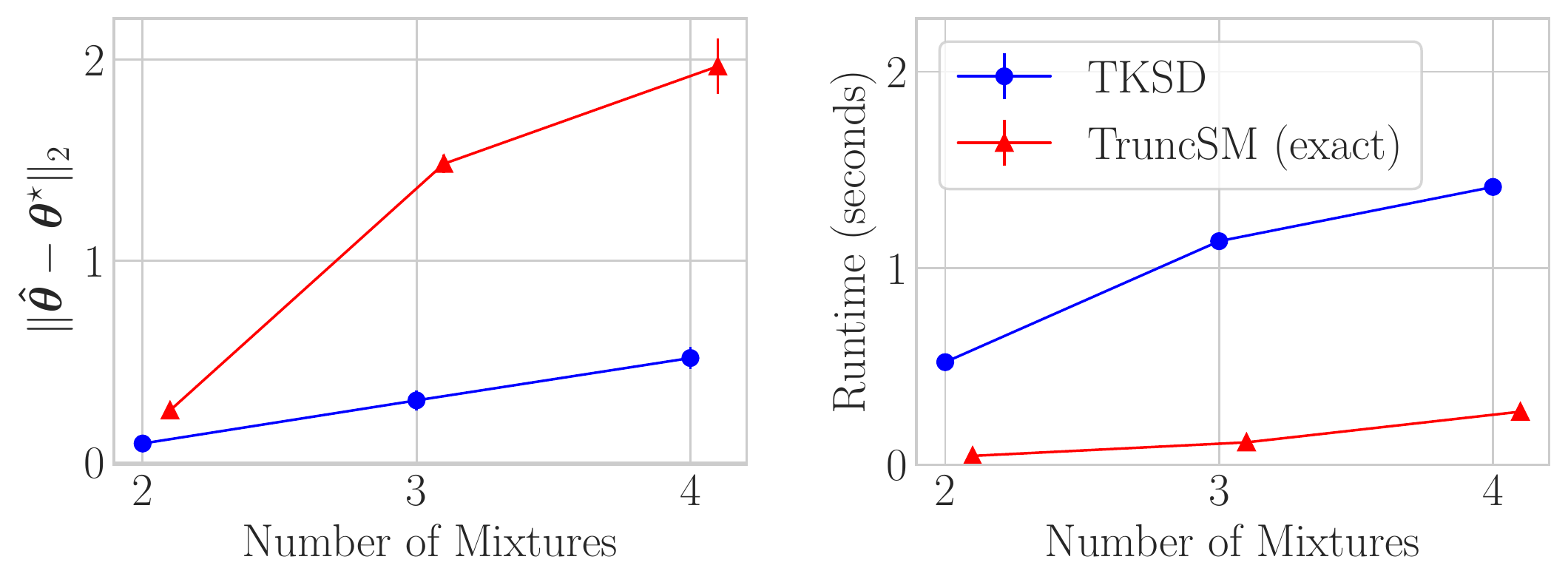}
\caption{Top: example setup for the experiment as the number of mixture modes increases from 2 to 4. Middle: estimation error as $n$ increases for 2 mixture modes, averaged across 256 seeds. Bottom: Estimation error as the number of mixture modes increases for a fixed $n=300$, averaged across 256 seeds.}
\label{fig:mixture}
\end{figure}

As an additional experiment to show the capability of the method, we test on a more complex problem, estimating several means of a Gaussian Mixture distribution. The estimation task is as follows. Fix $d=2$ and $m=200$. Let the mixture modes be given as follows,
\[
\mb\mu_1^* = [1.5, 1.5]\T, \;
\mb\mu_2^* = [-1.5, -1.5]\T,\;
\mb\mu_3^* = [-1.5, 1.5]\T,\;
\mb\mu_4^* = [1.5, -1.5]\T.
\]
We independently simulate samples from $\mathcal{N}(\mb\mu_i^*, I_d)$, for $i=1,\dots, 4$, and truncate these samples to within a box with vertices at $[-3, -3]$, $[3, -3]$, $[-3, 3]$ and $[3, 3]$
until we reach a total of $n$ samples after truncation. \Cref{fig:mixture}, top, shows an example of this experiment setup.

The task is to estimate $\mb\mu^*_i$ for all $i$. To ensure a well specified experiment, we set the initial conditions of the optimisation routine as a perturbation from the true value, i.e. $\mb\mu_{i}^{\text{ini}} = \mb\mu^*_i + \mb{z}$, $\mb{z} \sim \mathcal{N}(\mb{0}, 0.5 \cdot I_2)$. We estimate $\mb\mu$ with TKSD and compare it to a corresponding estimate by \textit{TruncSM} across a range of different values of $n$ and number of mixture modes, shown in \cref{fig:mixture}, middle and bottom. Overall, TKSD significantly outperforms \textit{TruncSM} in this experiment across all variations at the cost of runtime.

\subsection{Regression} \label{app:regression}

We provide a further example of using TKSD to estimate the parameters of a regression problem. We provide two examples in this section, for synthetic data and for real data.

\paragraph{Synthetic Data}

First, we simulate data in the following way:
\[
y_i \sim \mathcal{N}(\mu_i, 1), \;\; \mu_i = \beta_0^\star + \beta_1^\star x_i, \;\; x_i \sim \mathrm{Uniform}(0, 1)
\]
where we assume knowledge of the true values, $\beta_0^\star = 3$ and $\beta_1^\star = 4$. We truncate the dataset to where $y_i \geq 5 \; \forall i$, so only a portion of both $y$ and $x$ are observed. We then estimate the conditional density $p(y | x)$ with TKSD via estimation of $\mb\theta^\star \defeq [\beta_0^\star, \beta_1^\star]\T$. 

We obtain the log-likelihood of the Normal distribution under the estimate of $\hat{\mb\theta} \defeq [\hat\beta_0, \hat\beta_1]\T$ output by TKSD, split across the observed dataset (within the truncation boundary) and the unobserved dataset (outside the truncation boundary). We compare these log-likelihoods to ones obtained by a naive MLE approach which does not account for truncation. As an additional measure, we calculate the unobserved error, given by
%which is the mean squared error between the non-truncated and hence unobserved values of $\bm{y}$ (i.e. the data points that were truncated in the data simulation process, where $y_i < 5$) and their corresponding predictions, i.e.
\[
\mathrm{MSE}(\bm{y}, \hat{\bm{y}}) \defeq \sum^{n_{\mathrm{unobs}}}_{i=1} (\hat{y}_i - y_i)^2, \;\; \hat{y}_i = \hat\beta_0 + \hat\beta_1 x_i,
\]
where $n_{\mathrm{unobs}}$ is the number of \textit{discarded} data points due to truncation, and this error is calculated over these discarded data points only. This forms a measure which evaluates how well the density estimation method is predicting on unseen data points. Instead of traditional test error, it is the error on data points that have been truncated and hence no longer observed by either method.

A histogram of log-likelihoods and errors across 256 trials is given in \cref{fig:regressions} (left), and an example of one such simulated regression is given in \cref{fig:regressions} (top-right). The log-likelihoods for the observed points are higher for MLE, since the objective of this method is to directly maximise these likelihoods. However, on the unobserved dataset, the log-likelihoods obtained by TKSD are significantly higher than MLE. The unobserved errors are significantly smaller for TKSD, showing the improvement of TKSD over the naive approach.

\paragraph{Real Data: Student Test Scores}

We also experiment on a real-world dataset given by \citet{statatruncated} (Example 1). This dataset contains student test scores in a school for which the acceptance threshold is 40 out of 100, and therefore the response variable (the test scores) are truncated below by 40 and above by 100. Since no scores get close to 100, we only consider one-sided truncation at $y=40$. The aim of the regression is to model the response variable, the test scores, based on a single covariate of each students' corresponding score on a different test. \Cref{fig:regressions} (bottom-right) shows the plotted dataset and the regression line fit by TKSD and naive MLE. Whilst we have no true baseline value to compare to, the TKSD regression line seems to account for the truncation, whilst as expected, MLE does not.

\begin{figure}[t!]
\centering
\begin{tabular}{cc}
\includegraphics[height=10cm]{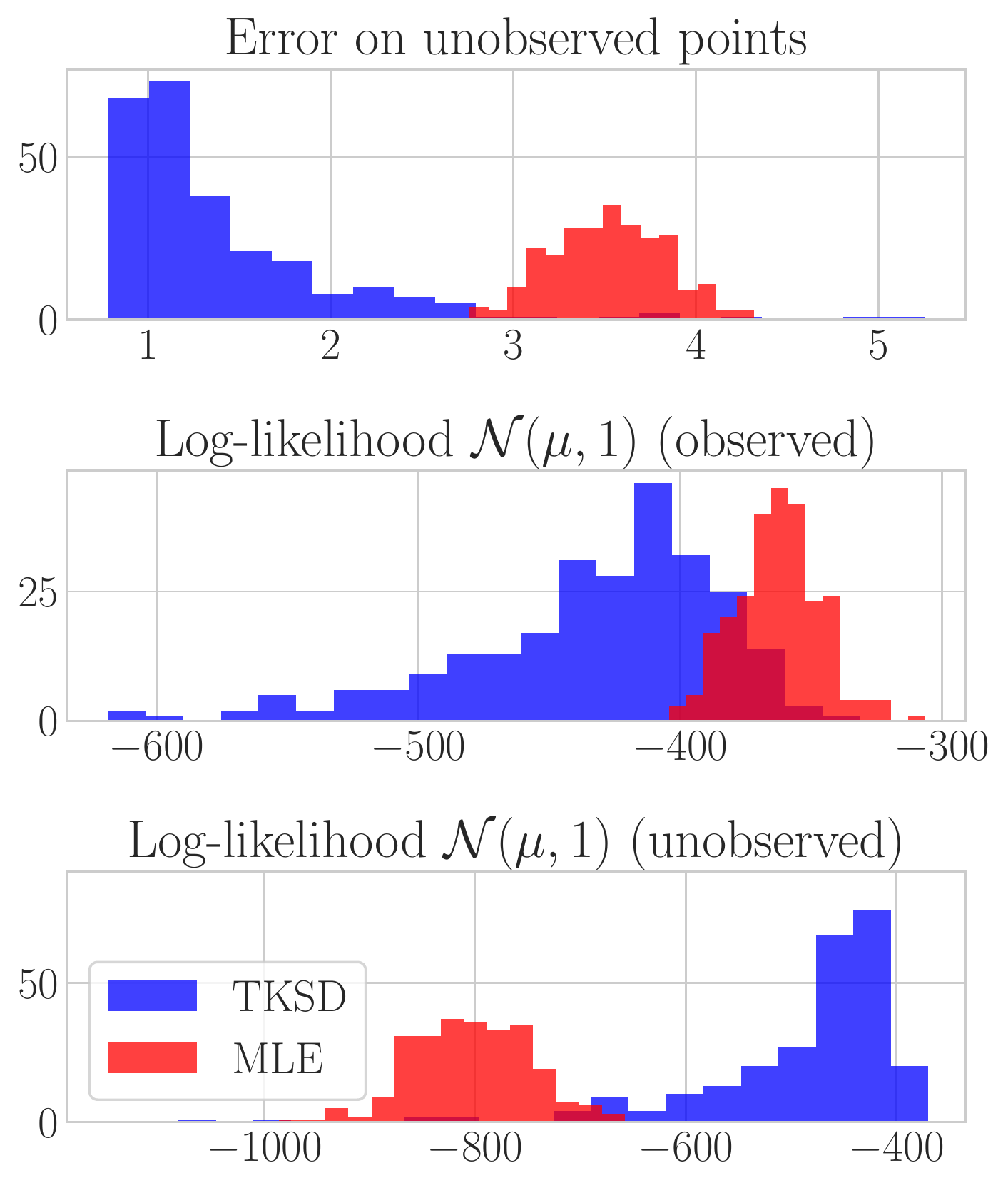}&
\begin{tabular}[b]{c}
\includegraphics[height=5cm]{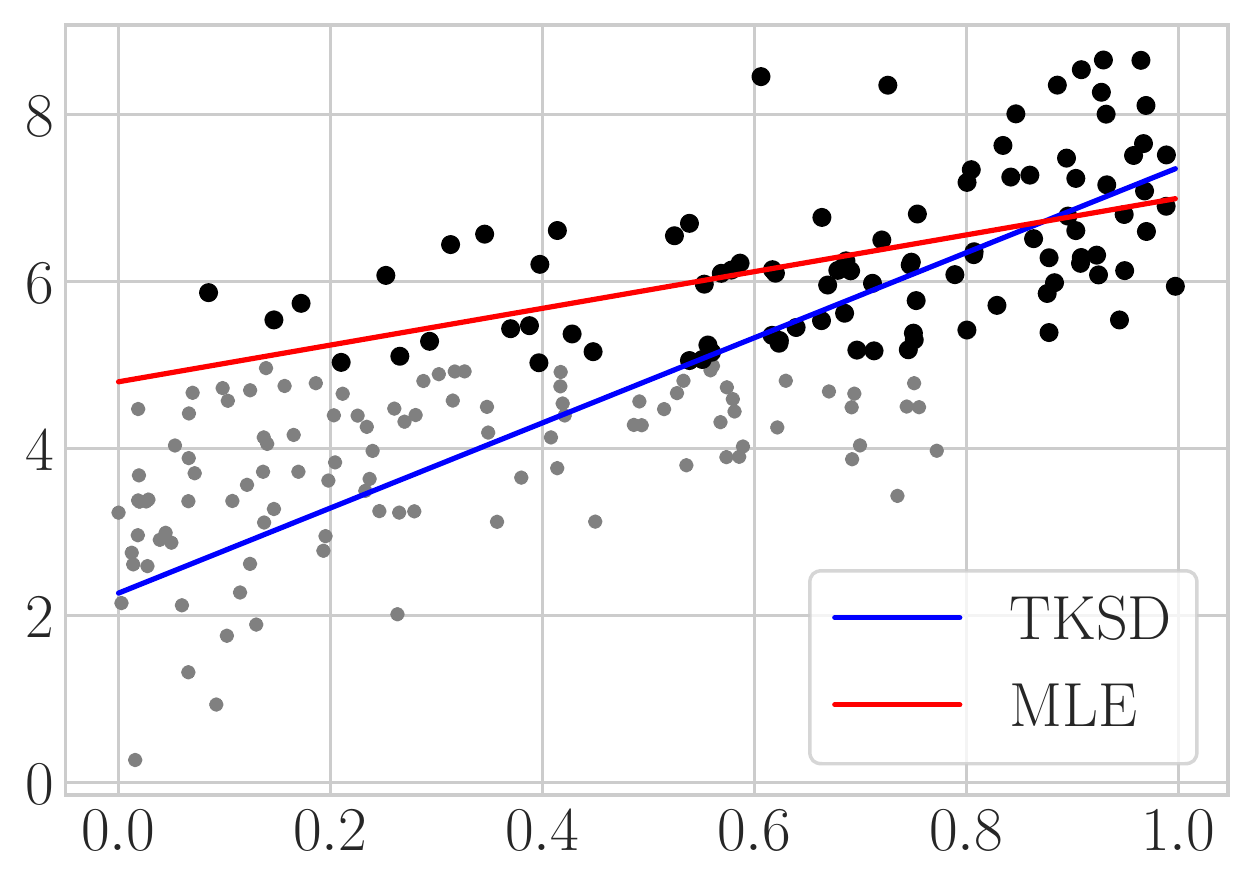} \\
\includegraphics[height=5cm]{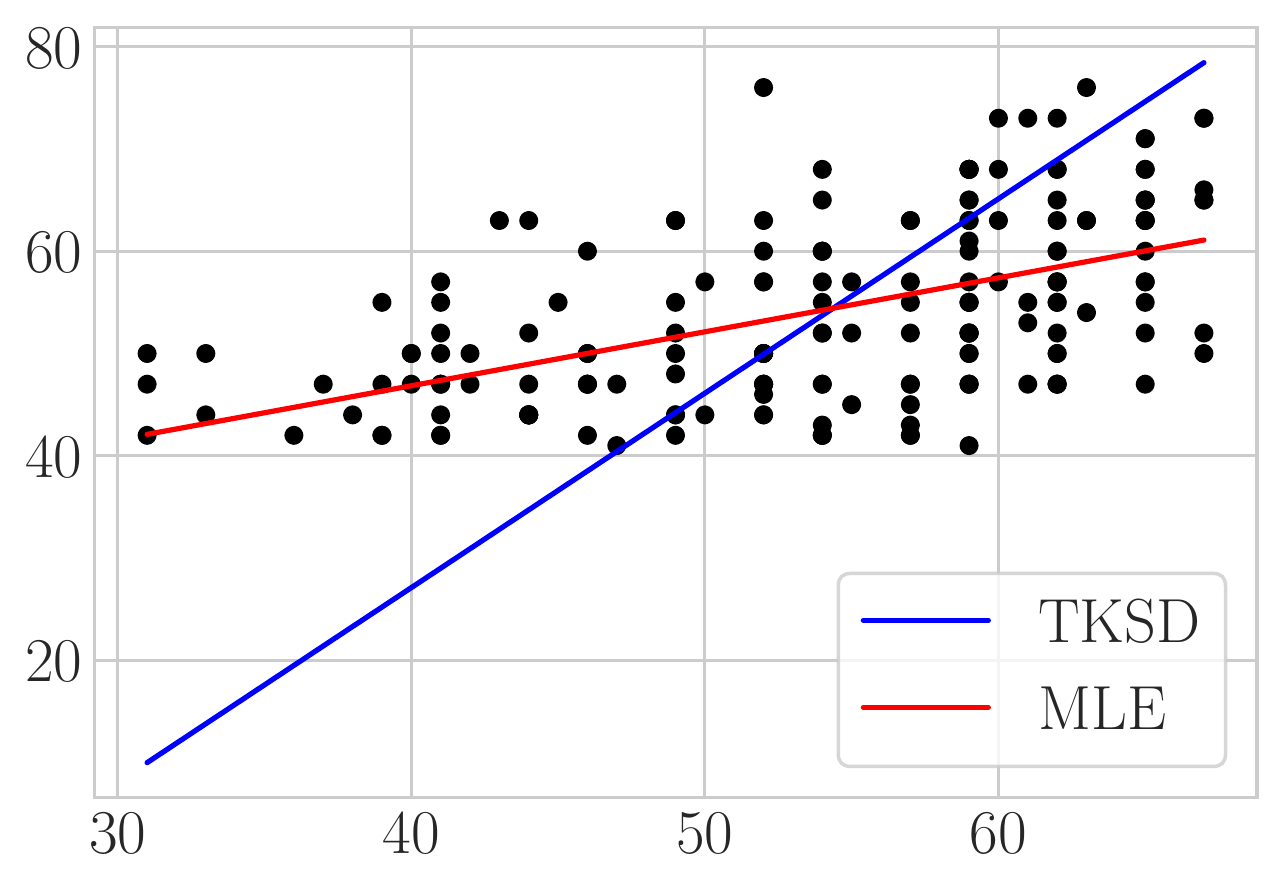}
\end{tabular} \\
\end{tabular}
\caption{Statistics from regression fit. Left:Mean squared errors only on the observed dataset (top). Log-likelihoods on the observed dataset (middle) and the unobserved dataset (bottom), both for a $\mathcal{N}(\mu, 1)$ distribution and over 256 trials. Right: An example of TKSD used to fit regression on the synthetic data on one trial (top) and TKSD used to fit regression on student test scores data (bottom). Fuller black points are the observed data points, \textit{after} truncation, and smaller grey points are the unobserved data points that were truncated out in the data simulation process.}
\label{fig:regressions}
\end{figure}

\subsection{Quantifying Effect of Boundary Point Distribution} \label{app:boundary_dist}

We test whether the effect of boundary point sampling distribution has an effect on the robustness of TKSD. To do so, we repeat the simple experiment setup where data are simulated as follows,
\[
\mb{x} \sim \mathcal{N}(\mb{\mu}^\star, \mb{I}_2), \; \mb{\mu}^\star = \begin{bmatrix}
    1, & 1
\end{bmatrix}\T
\]
from which we observe $n=400$ realisations of $\mb{x}$ that are restricted to the unit $\ell_2$ ball around the origin, and let $m=30$. We use TKSD to provide an estimate $\mb{\hat{\mu}}$ of $\mb{\mu}^\star$ under three scenarios:
\begin{enumerate}
\itemsep0em
    \item Boundary points are distributed towards $\mb{\mu}^\star$, i.e. samples from $\approxdV$ are closer to the centre of the dataset.
    \item Boundary points are distributed away from $\mb{\mu}^\star$, i.e. samples from $\approxdV$ are closer to the edge of the dataset.
    \item Boundary points are sampled uniformly across $\partial V$.
\end{enumerate}

\begin{figure}[t!]
    \centering
    \includegraphics[width=\linewidth]{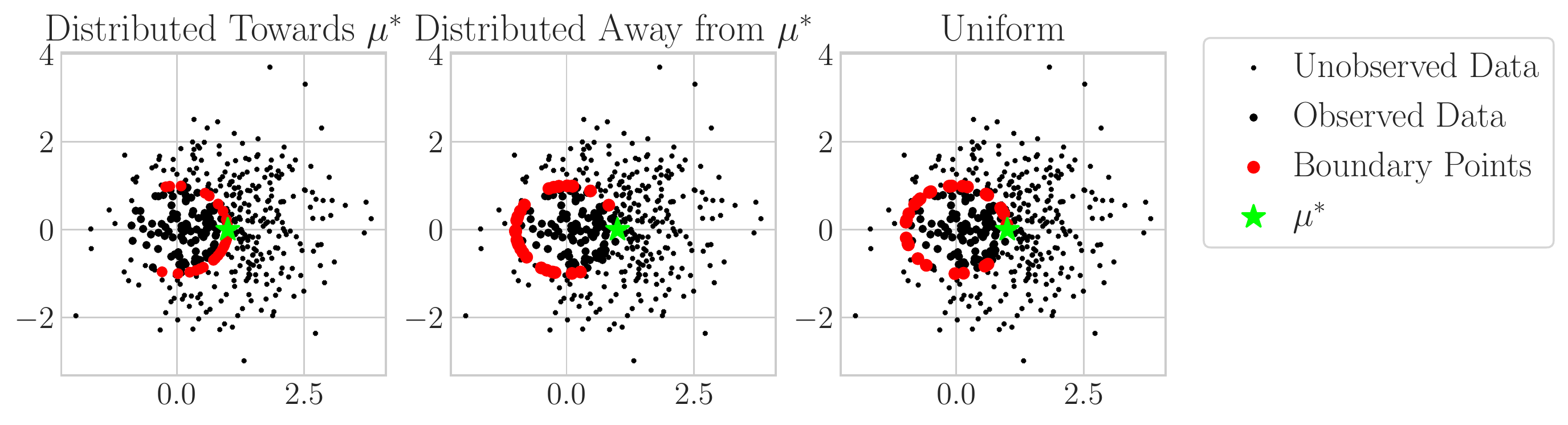}
    \caption{Example of experiment setup for measuring the effect of the boundary point distribution.}
    \label{fig:boundary_dist_example}
\end{figure}

\begin{figure}[t!]
    \centering
    \includegraphics[width=0.5\linewidth]{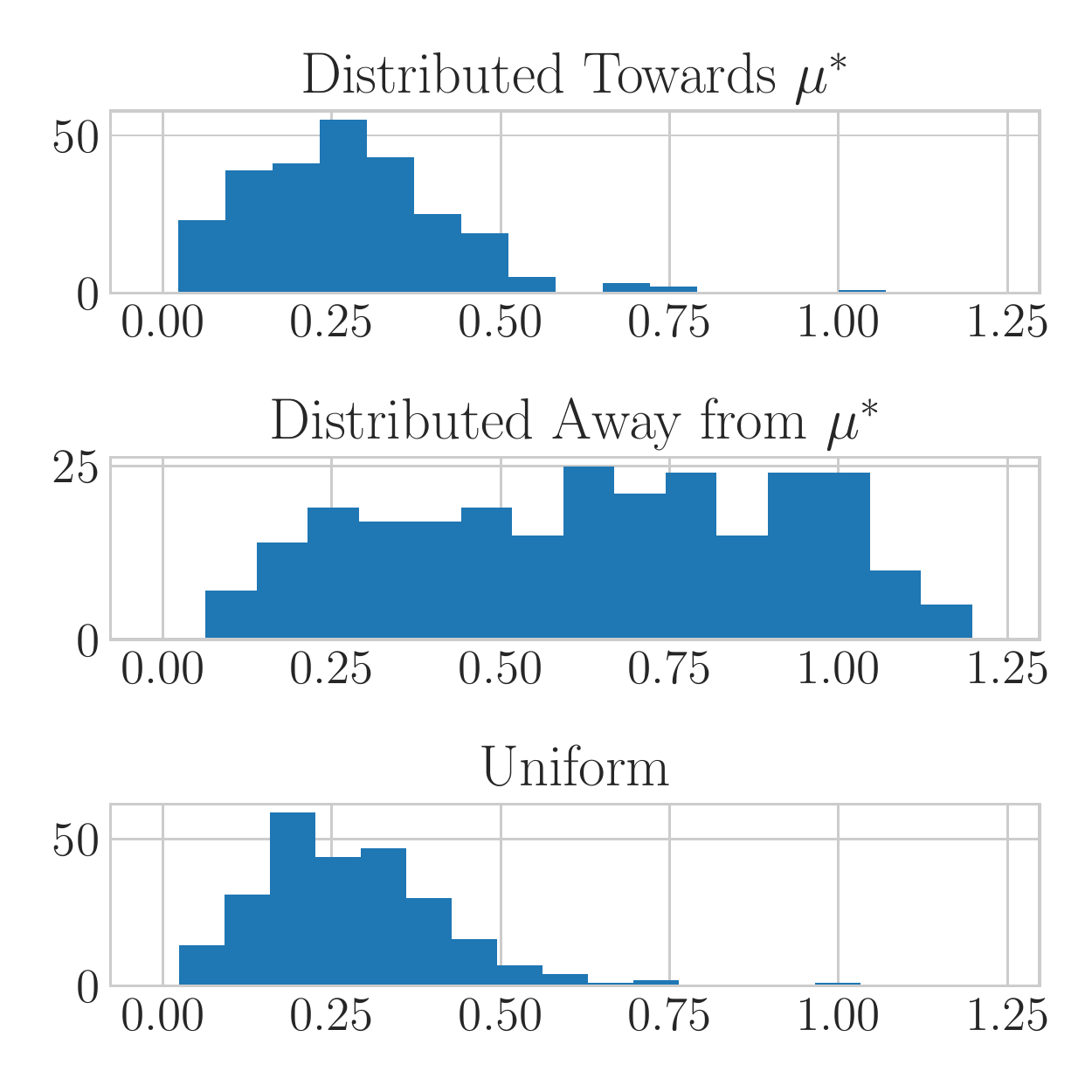}
    \caption{Frequency of estimation errors ($\|\mb{\hat{\mu}} - \mb{\mu}^\star\|$) in the simple experiment setup for the three differently distributed boundary points.}
    \label{fig:boundary_dist_hists}
\end{figure}

See \cref{fig:boundary_dist_example} for a visual representation of these three scenarios. We measure the estimation error, $\|\mb{\hat{\mu}} - \mb{\mu}^\star\|$, and test whether one scenario provides a lower error on average overall. \Cref{fig:boundary_dist_hists} shows the distribution of all three estimation errors across 256 trials. Scenario 1 and 3 provide comparable error distributions which are significantly smaller than the error distribution of scenario 2, implying that either the boundary needs to be covered fully, or the boundary points need to be `representative' in some way, where the most significant truncation effect is.

\subsection{Choosing Boundary Size for Dimension Benchmarks}\label{app:bsize}

\begin{figure}[ht!]
    \centering
    \includegraphics[width=0.45\linewidth]{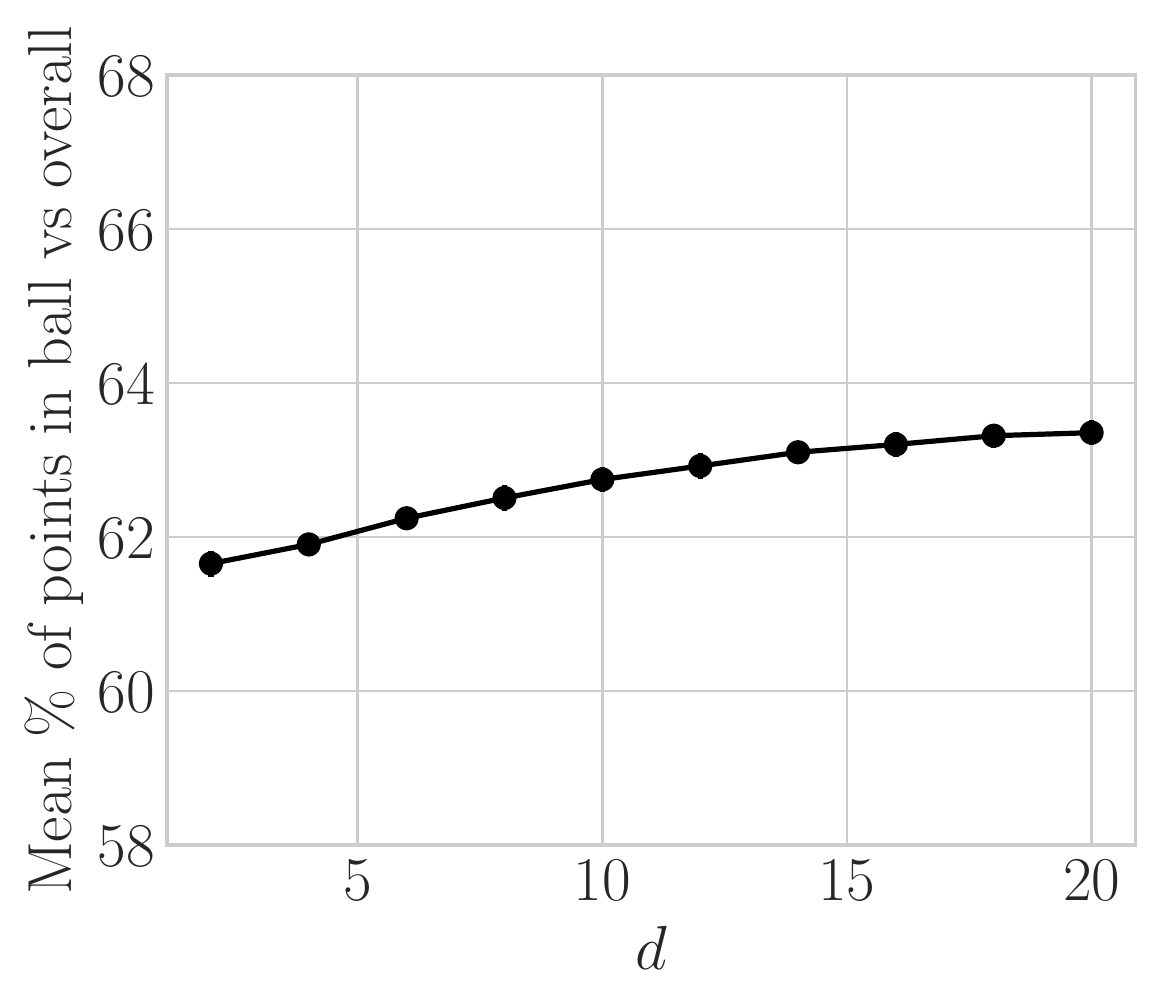}
    \includegraphics[width=0.45\linewidth]{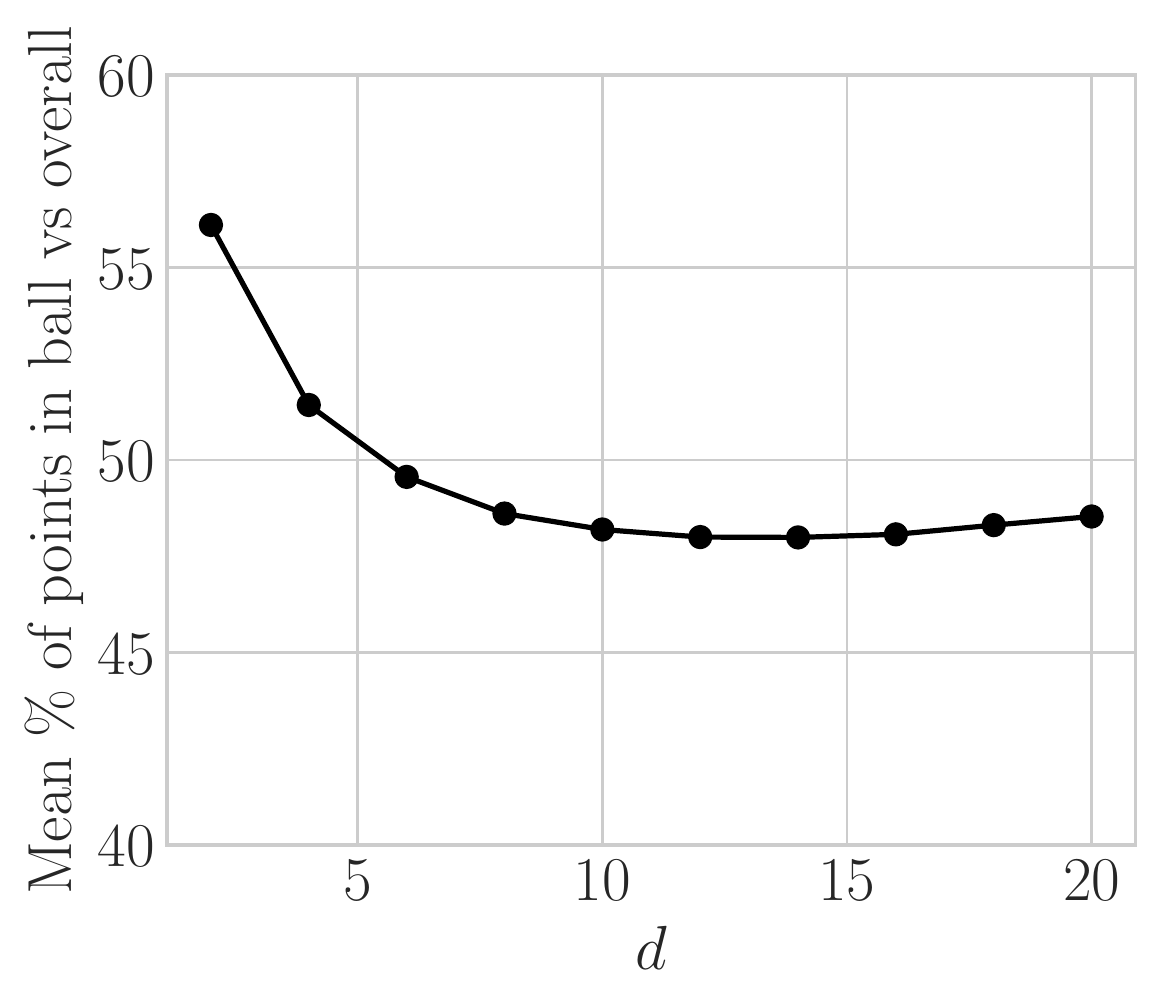}
    \caption{Mean percentage of points which remain after truncation as dimension $d$ increases, where the truncation domain is a $\ell_c$ ball of radius $d^{0.98}$ for $c=1$ (left), and $d^{0.53}$ for $c=2$ (right).}
    \label{fig:ballsizes}
\end{figure}

In \cref{sec:dimensionbench}, we chose the size of the boundary to scale with $d$ so that roughly the same amount of data points are truncated for each value of dimension $d$. The values of $d^{0.98}$ and $d^{0.53}$ for the $\ell_1$ ball and $\ell_2$ ball respectively were chosen via trial and error such that the percentage of points simulated from the Normal distribution remaining after truncation did not vary significantly. Figure \ref{fig:ballsizes} shows that with this choice of $\ell_1$ and $\ell_2$ ball radius, the mean percentage of points that remain after truncation remains at consistent as dimension increases in both cases.

\subsection{Extra Computation Details}\label{app:extracomp}

\begin{itemize}
    \item The inversion of $\bm{K}'$ is the largest computational expense when it comes to evaluating the $U$-statistic or $V$-statistic (\eqref{eq:h-ustatistic} and \eqref{eq:Vstatistic} respectively).
    
    If $m < d$, which will be common as we want $m$ to be large, then $\bm{K}'$ is rank deficient. To circumvent this issue, we invert $\bm{K}' + \epsilon \bm{I}_m$ instead, where $\epsilon>0$ is small. Additionally, $\bm{K}' + \epsilon \bm{I}_m$ is symmetric and positive definite, so we can exploit its Cholesky decomposition to make the inversion faster and more stable. 
    
    Overall, this inversion looks like
    \[
    (\bm{K}')^{-1} \approx (\bm{K}' + \epsilon \bm{I}_m)^{-1} = \bm{L}^{-1}(\bm{L}^{-1})\T,
    \]
    where $\bm{L}$ is the corresponding lower triangular matrix from the Cholesky decomposition of $\bm{K}' + \epsilon \bm{I}_m$. The inversion of $\bm{L}$, a lower triangular matrix, requires half as many operations as inverting $\bm{K} + \epsilon \bm{I}_m$ directly \citep{cholesky}.

\item We also make use of methods presented in \citet{jitkrittum2017} for fast computation when constructing all kernel matrices in Python.

\end{itemize}
\nocite{jitkrittum2017}
%%%%%%%%%%%%%%%%%%%%%%%%%%%%%%%%%%%%%%%%%%%%%%%%%%%%%%%%%%%%%%%%%%%%%%%%%%%%%%%
%%%%%%%%%%%%%%%%%%%%%%%%%%%%%%%%%%%%%%%%%%%%%%%%%%%%%%%%%%%%%%%%%%%%%%%%%%%%%%%

\end{document}